\PassOptionsToPackage{table,xcdraw}{xcolor}
\documentclass[sn-mathphys]{sn-jnl}

\normalbaroutside

\usepackage{url,lineno,microtype,subcaption}
\usepackage{makeidx}
\usepackage{import}
\usepackage{enumerate}
\usepackage{threeparttable,booktabs,arydshln}
\usepackage{mathrsfs}
\usepackage[utf8]{inputenc}
\usepackage[english]{babel}
\usepackage{sidecap}
\usepackage{bigints,graphics,amsfonts,amsmath,amssymb}

\usepackage{yhmath}
 \usepackage{afterpage}

\makeatletter
\newcommand{\subfiglabel}[2]{%
	\protected@write \@auxout {}{\string \newlabel {#1}{{\thefigure #2}{\thefigure}{\thefigure #2}{#1}{}} }%
	\hypertarget{#1}{}%
}
\makeatother

\usepackage{graphicx}

\usepackage[numbers]{natbib}

\usepackage[font=small,labelfont=bf]{caption}
\usepackage{times}
\usepackage{multicol}
\usepackage{multirow}
\usepackage{hhline}
\usepackage{wrapfig}
\usepackage{xcolor}

\usepackage{colortbl}
\usepackage{color}
\usepackage{tabularx}

\newcommand{\changed}[1]{{#1}} %{{\color{black}{#1}}}
\newcommand{\changedB}[1]{{#1}} %{{\color{blue}{#1}}}
\newcommand{\changedC}[1]{{#1}} % {{\color{blue}{#1}}}

\newcommand{\changedSB}[1]{{#1}} % {{\color{violet}{#1}}}

\newcommand{\changedXW}[1]{{#1}} % {{\color{blue}{#1}}}
\newcommand{\changedF}[1]{{#1}} % {{\color{violet}{#1}}}
\newcommand{\revisedXW}[1]{{#1}} % {{\color{blue}{#1}}}

\definecolor{darkblue}{rgb}{0.1,0.1,0.6}
\newcommand{\highlight}[1]{{#1}} % {{\color{darkblue}{#1}}}

\newcommand{\subsubsectionB}[1]{\subsubsection{#1}}

\newtheorem{lemma}{Lemma}
\title[Article Title]{Article Title}

\begin{document}

%\onecolumn
%\firstpage{1}

\title[Coordination-free Planning for Congestion Reduction]{Coordination-free Multi-robot Path Planning for Congestion Reduction Using Topological Reasoning} 

%\author[\firstAuthorLast ]{\Authors} %This field will be automatically populated
%\address{} %This field will be automatically populated
%\correspondance{} %This field will be automatically populated

%\extraAuth{}% If there are more than 1 corresponding author, comment this line and uncomment the next one.
%\extraAuth{corresponding Author2 \\ Laboratory X2, Institute X2, Department X2, Organization X2, Street X2, City X2 , State XX2 (only USA, Canada and Australia), Zip Code2, X2 Country X2, email2@uni2.edu}

\author[1]{\fnm{Xiaolong} \sur{Wang}}\email{wangxiaolong0830@hotmail.com}

\author[1]{\fnm{Alp} \sur{Sahin}}\email{als421@lehigh.edu}

\author*[1]{\fnm{Subhrajit} \sur{Bhattacharya}}\email{sub216@lehigh.edu}

\affil*[1]{ %Department of Mechanical Engineering and Mechanics, Lehigh University, 19 Memorial Drive West, Bethlehem, PA 18015, U.S.A.
	\orgdiv{Department of Mechanical Engineering and Mechanic}, \orgname{Lehigh University}, \orgaddress{\street{19 Memorial Drive West}, \city{Bethlehem}, \postcode{18015}, \state{PA}, \country{U.S.A.}}}

%\affil[2]{\orgdiv{Department}, \orgname{Organization}, \orgaddress{\street{Street}, \city{City}, \postcode{10587}, \state{State}, \country{Country}}}

%\affil[3]{\orgdiv{Department}, \orgname{Organization}, \orgaddress{\street{Street}, \city{City}, \postcode{610101}, \state{State}, \country{Country}}}

%\begin{abstract} 
\abstract{
    \changedSB{We consider the problem of multi-robot path planning in a complex, cluttered environment with the aim of reducing overall congestion in the environment, while avoiding any inter-robot communication or coordination. 
    Such limitations may exist due to lack of communication or due to privacy restrictions (for example, autonomous vehicles may not want to share their locations or intents with other vehicles or even to a central server).
    \changedB{The key insight that allows us to solve this problem is to stochastically distribute the robots across} different \changedB{routes in the environment}
    by assigning them paths in different topologically distinct classes, 
    \changedSB{so as to} lower congestion and the \changedB{overall travel time for all robots in the environment}.
    We outline the computation of topologically distinct paths in a spatio-temporal configuration space and propose methods for the stochastic assignment of paths to the robots.
    A fast replanning algorithm and a potential field based controller allow robots to avoid collision with nearby agents while following the assigned path.
    Our simulation and experiment results show a significant advantage over shortest path following under such a coordination-free setup.
    \footnote{
%    	A preprint version of this articles is posted on the arXiv preprint repository and can be accessed at \url{https://arxiv.org/abs/2205.00955}~\cite{coordination-free-arxiv:22}.
    	This version of the article has been accepted for publication, after peer review but is not the Version of Record and does not reflect post-acceptance improvements, or any corrections. The Version of Record is published in Journal of Intelligent \& Robotic Systems, and is available online at: \url{http://dx.doi.org/10.1007/s10846-023-01878-3}.
    }
}

%\tiny
%\keyFont{ \section{Keywords:} Path Planning for Multiple Mobile Robots or Agents, Multi-Robot Systems, Motion and Path Planning, Multi-robot Coordination, Topological Path Planning}
%\end{abstract}
}

\keywords{
%	Path Planning for Multiple Mobile Robots or Agents, Multi-Robot Systems, Motion and Path Planning, Multi-robot Coordination, Topological Path Planning
Multi-Robot Motion Planning, Topological Path Planning, Privacy-aware Planning
}

\maketitle

\section{Introduction}

\subsection{Motivation and Problem Description}

\revisedXW{Autonomous vehicles are expected to travel in urban environments in the future for increased safety and overall efficiency. They could be of different car brands running different navigation \& communication systems that do not share their route-choosing processes or travel data, either due to lack of communication or due to privacy restrictions.
To avoid traffic congestion, such as those caused by non-cooperating human drivers nowadays, independent autonomous vehicles need to have a method for distributing traffic in the environment without communication. Motivated by this real-world scenario, in this paper}
we consider the problem of path planning for a large number of privacy-aware robots in a complex, cluttered \changedF{indoor or urban} environment with uncertainties (other unpredictable agents such as pedestrians), where the robots need to be well-distributed throughout the environment and avoid congestion in any region, but are not allowed to communicate or share their location data or intents with other robots. 
This is relevant to {avoiding congestion in distributed vehicle routing problems when a vehicle's location/intent cannot be shared \revisedXW{either due to lack of communication or} 
due to privacy restrictions}. We address the fundamental question of 

 \setcounter{figure}{0}
\setcounter{subfigure}{0}
% \begingroup
\setlength{\columnsep}{5pt}%
\begin{wrapfigure}{r}{0.3\textwidth}
	\vspace{-2em}
	\begin{center}
		\includegraphics[width=0.3\textwidth]{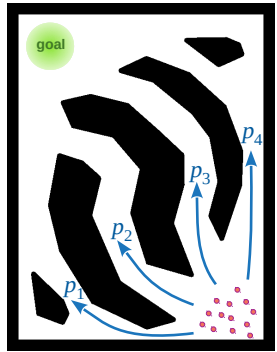}
	\end{center}
	\caption{Key idea: Make each robot stochastically choose from a set of topologically distinct paths in the environment.} \label{pic_insight}
	\vspace{-4em}
\end{wrapfigure}
\noindent
path planning under such circumstances without any inter-robot coordination, while trying to minimize the overall congestion in the environment.
We assume that each
robot knows the map of the environment and can localize itself in it, \revisedXW{but do not know the location or intent of the other robots}.
Furthermore, a robot can detect other agents in its immediate neighborhood (for example using on-board cameras or laser range sensors) so as to be able to avoid \revisedXW{immediate} collisions with them, although they cannot broadcast or communicate any information with each other.
%\cite{coordination-free-arxiv:22}
\revisedXW{The key insight that allows us to solve this problem is to make each robot stochastically choose from a set of different \changedB{routes in the environment}
    which correspond to paths in different topologically distinct classes (Figure~\ref{pic_insight}), 
    \changedSB{so as to} lower congestion and the \changedB{overall travel time for all robots in the environment}.}

% We assume that there are two types of agents in the environment -- \emph{i.} rational agents (also referred to as \emph{robots}) that try to altruistically minimize the global traffic congestion, and, \emph{ii.} non-rational agents (also referred to as \emph{pedestrians}) that try to minimize their own travel time only. All agents know the map of the environment (and hence know the location of the static obstacles) and are able to localize themselves in the map. They can also detect other agents in their immediate neighborhood (for example using on-board cameras or laser range sensors) so as to be able to avoid collisions with them, although they cannot broadcast or communicate any information with each other.

\subsection{Literature Review}

\noindent\changedF{\emph{MAPF: }}
\changedSB{%The problem of \changedSB{path} planning in a dynamic, cluttered and uncertain environment is well-studied~\citep{MOHANAN2018171}. 
	\changedXW{Multi-agent path finding (MAPF) is a well-studied problem with a variety of practical applications. 
		% In such a problem, we are usually given a neat environment and a set of agents, each one of them assigned with a start and a goal position. 
		\changedF{Given a set of robots and their start locations, the objective of MAPF is}
		to find a set of paths that lead them to their corresponding goals without collision, while minimize the sum of travel time of each agent. Early studies \citep{silver2005cooperative} \changedF{focused on computing} valid collision-free \changedF{solutions}, while recent method \citep{sharon2015conflict} and its many variants \changedF{have strived to compute} optimal solutions. \highlight{They all focus on conflict \changedF{resolution} among the agents' path choices so that \changedF{an} optimal or near-optimal solution can be achieved.}}
	\changedXW{\changedF{Furthermore,} existing algorithms \citep{han2020ddm,wagner2011m} are designed to deal with
		% trajectories that are not strongly coupled, which means the initial and the destination positions are mostly 
		\changedF{agents that are well-}distributed across the environment \highlight{and not have similar start/goal locations. \revisedXW{For example, in ~\cite{han2020ddm}, in all presented results, the initial robot positions are chosen to be distributed uniformly throughout the environment.} For a large group of robots with similar start and goal points, the computation involved in existing MAPF methods will increase drastically due to increased collision avoidance computations, and the collision/congestion avoidance still happens at a local level}.
		% ; otherwise, their computation time would increase drastically for optimally navigating a group of robots with similar start and goal points.
		% Although some noticed diversification in the planning, it is still at a local level for fast swaps of positions.
	}
	
	MAPF in dynamic, cluttered and uncertain environments is also well-studied in robotics~\citep{MOHANAN2018171,Kushleyev:09,2020arXiv200614195C,Hasan:18}.}
Such dynamic environments exist in presence of pedestrians \changedSB{and other robots} in busy indoor or urban environments. Robots employed in such environments need to arrive at designated target locations while avoiding both static and dynamic obstacles such as pedestrians, which is usually unknown to the path planner at the first place.
\revisedXW{The dynamic nature of the environment in all these existing work, however (see the review paper~\cite{2020arXiv200614195C} for example), is assumed to arises from completely unpredictable agents (both human pedestrians and other robots), and hence these work focus primarily on improving the short-term prediction of the behavior of such agents without consideration for long-term congestion reduction. We, on the other hand, consider the problem where the robots in the environment use the same stochastic algorithm that, even without inter-agent coordination or communication, results in overall reduction of congestion in the environment in the long-term.}

\revisedXW{Along similar lines, some studies~\citep{ziebart2009planning,kretzschmar2016socially} have focused more on pedestrian's trajectory prediction or social compliance with humans that facilitate obstacle avoidance in a more human-friendly way, but is still at a local level and over short time horizons. Over longer time periods, when multi-robot groups run on a large complicated map cluttered by dynamic obstacles, this does not help proactively avoid robot congestion in the long run if the robots' routes are not well-distributed across the environment. In this paper we focus on developing algorithms for individual robots that, even without inter-robot communication or coordination, tries to reduce overall congestion in the environment by stochastically distributing the robots over different routes.}

While the presence of uncertain agents such as pedestrians in such planning problems have been considered~\citep{kirby2009companion, sisbot2007human,ferrer2013social,shiomi2014towards}, the problem of reducing robot traffic congestion by taking advantage of the structure of the environment without explicit inter-robot coordination remains open.

\vspace{0.5em}
\noindent\changedF{\emph{Congestion Avoidance in MAPF: }}
\changedSB{Both with and without uncertainties, existing literature uses inter-robot coordination for global planning to avoid congestion~\citep{atzmon2020probabilistic,santos2011hierarchical,toutouh2018swarm,ANTONIOU20131167}, or use information stored in the environment (represented as a network) as a means of indirect coordination between the agents~\citep{1706747,gunes2002ara,di2005anthocnet}.
	\revisedXW{A recent work on congestion-aware policy synthesis \citep{street2021congestion} is notable, and \changedF{tries to achieve a balance between} congestion \changedF{reduction} and \changedF{minimization of detours by designing single-robot automata}. \changedF{However, even in this approach, significant centralized inter-robot coordination (using a shared \emph{probabilistic reservation table}) is necessary. %Furthermore, this work does} not consider out-of-system agents (unpredictable agents such as pedestrians). 
 This assumption of inter-robot coordination is fundamentally different from the premise of our current work, where we assume that there exists no inter-agent communication or coordination, and robots do not share their plan or intent.
% The algorithm that we propose does not require the robots to communicate or coordinate about their strategies, neither does .
Furthermore, \citep{street2021congestion} does not consider out-of-system agents (unpredictable agents such as pedestrians), which we do in our current work.
 }}
	
	\vspace{0.5em}
	\noindent\changedF{\emph{Other Related Multi-robot Coordination Problems: }}
	Most multi-robot persistent patrolling/surveillance methods use some centralized coordination~\citep{Stump:11:surveillance,Leahy:Surveillance:16,Kusnur_muav_19,Thakur-2013-109538}.
	\changedXW{In a inter-robot communication denied situation, local methods \highlight{(\emph{e.g.}, repulsive force or velocity obstacle \citep{van2008reciprocal})} or other decentralized framework \highlight{(using topological braids \citep{mavrogiannis2019multi} \revisedXW{or rotations \citep{mavrogiannis2022winding}})} can only help avoid collision but not avoid congestion proactively.
		% In those local methods, safety distance is usually not taken into account and agents keep their velocities high until a collision is about to happen. However, in reality, agents (especially human involved ones) would always slow down when entering into a busier area due to a longer safety distance especially in their forward directions.
		% As a result, a congestion could readily happen without inter-robot coordination and cause a worse performance of such methods than expected.
		\changedF{In those methods, without inter-robot coordination, congestion could readily happen.}
	}
	Action strategy planning for robots minimizing expected cost, %even for multi-robot systems, 
	while a well-researched area (and often addressed using reinforcement learning or game-theoretic methods~\citep{lavalle2000robot,he2020integral}), %existing methods 
	mostly focuses on local actions, do not take global topology of the environment into account, rely on shared information, or do not scale with the number of robots \revisedXW{(for example, \cite{he2020integral} presents results with only three robots)}.
}

\subsection{Solution Overview and Paper Outline}
\revisedXW{The technical problem considered in this paper can be summarized as follows:}

\vspace{0.5em}
\begin{minipage}{0.9\textwidth}
    \revisedXW{\textbf{Problem Statement:} Given a discrete graph representation of an environment, how can a large number of privacy-aware robots with local sensing plan their respective paths in the graph without inter-robot communication or centralized coordination (\emph{i.e.}, without sharing their location or intents) so as to minimize the overall congestion in the environment.}
\end{minipage}
\vspace{0.5em}

\noindent
If all robots choose the same/similar paths (\emph{e.g.}, shortest paths), %that will inevitably result in 
certain regions of the environment will inevitably be traversed more. \changedF{This issue is even more aggravated when the robots or groups of robots have similar start and goal locations}.
%The fundamental premise behind this task is that 
%such choices, under the said scenario, reduce to a stochastic choice between the different topological classes of paths that an agent can adopt.
Instead, a robot stochastically chooses between different topologically distinct classes of solutions with an aim of reducing overall congestion in the environment and to altruistically minimize overall travel time for any robot.
\revisedXW{We leverage the topological path planning methods introduced in the author's prior works \citep{planning:AURO:12,Homotopy-Planning:journal:18,Wang:Cable-Controlled:2018,DARS:14:HRI,Persistence-Plnning:TRO:15,ICRA:20:path:deconfliction}
for computing the topologically distinct paths in the environment.

\vspace{0.2em}\noindent
\textbf{Contributions:}
The main new contributions of the paper are:
\begin{enumerate}
    \item Formulating an optimization problem for coordination/communication-free computation of probability values using which a robot stochastically chooses a path out of the available topologically distinct choices.
    \item Develop efficient computation methods for solving the said optimization problem by employing simplifications.
    \item Design the controller used by each non-holonomic robot to follow their chosen path while locally avoiding collisions with other agents in its immediate neighborhood. This requires the development of a fast re-planner that takes the topological class into account.
    \item Validate the proposed method in simulations and real robot experiments.
\end{enumerate}
% robots are able to exploit the topological structure of the environment.
% Using our proposed approach, robots distribute themselves across the different classes (routes) without %knowing what paths other agents are going to follow, and 
% %with no 
% any
% global coordination.
}

% \begin{minipage}{\linewidth}
% \vspace{0.2em}
% \begin{wrapfigure}{r}{0.5\textwidth}
%   \begin{center}
%     \includegraphics[width=0.48\textwidth]{pipeline_new.pdf}
%   \end{center}
%   \caption{Birds}
% \end{wrapfigure}

% \revisedXW{\textbf{Paper Outline:}}
% In Section~\ref{sec:top-planning} we provide background on the computation of paths in topologically distinct classes using discrete search in $\mathbb{Z}_2$-homology augmented graphs.
% % 
% \revisedXW{The main contributions of the paper appear in} Section~\ref{sec:path-combination-evaluation}, where we describe the coordination-free computation 
% % (\emph{i.e.}, without communication with other agents in the environment and without knowing their locations or intents) 
% of the probability values using which a robot stochastically chooses a path out of the available topologically distinct choices.

% \end{minipage}

\vspace{0.5em}\noindent
\revisedXW{\textbf{Paper Outline:}}
In Section~\ref{sec:top-planning} we provide background on the computation of paths in topologically distinct classes using discrete search in $\mathbb{Z}_2$-homology augmented graphs.
\revisedXW{The main contributions of the paper appear in} Section~\ref{sec:path-combination-evaluation}, where we describe the coordination-free computation 
% (\emph{i.e.}, without communication with other agents in the environment and without knowing their locations or intents) 
of the probability values using which a robot stochastically chooses a path out of the available topologically distinct choices.
This includes estimation of traffic density in the environment (Section~\ref{sec:traffic-density}), which is used 

% \setcounter{figure}{0}
%\setcounter{subfigure}{0}
% \begingroup
\setlength{\columnsep}{5pt}%
\begin{wrapfigure}{r}{0.58\textwidth}
\vspace{-0.5em}
  \begin{center}
    \includegraphics[width=0.58\textwidth]{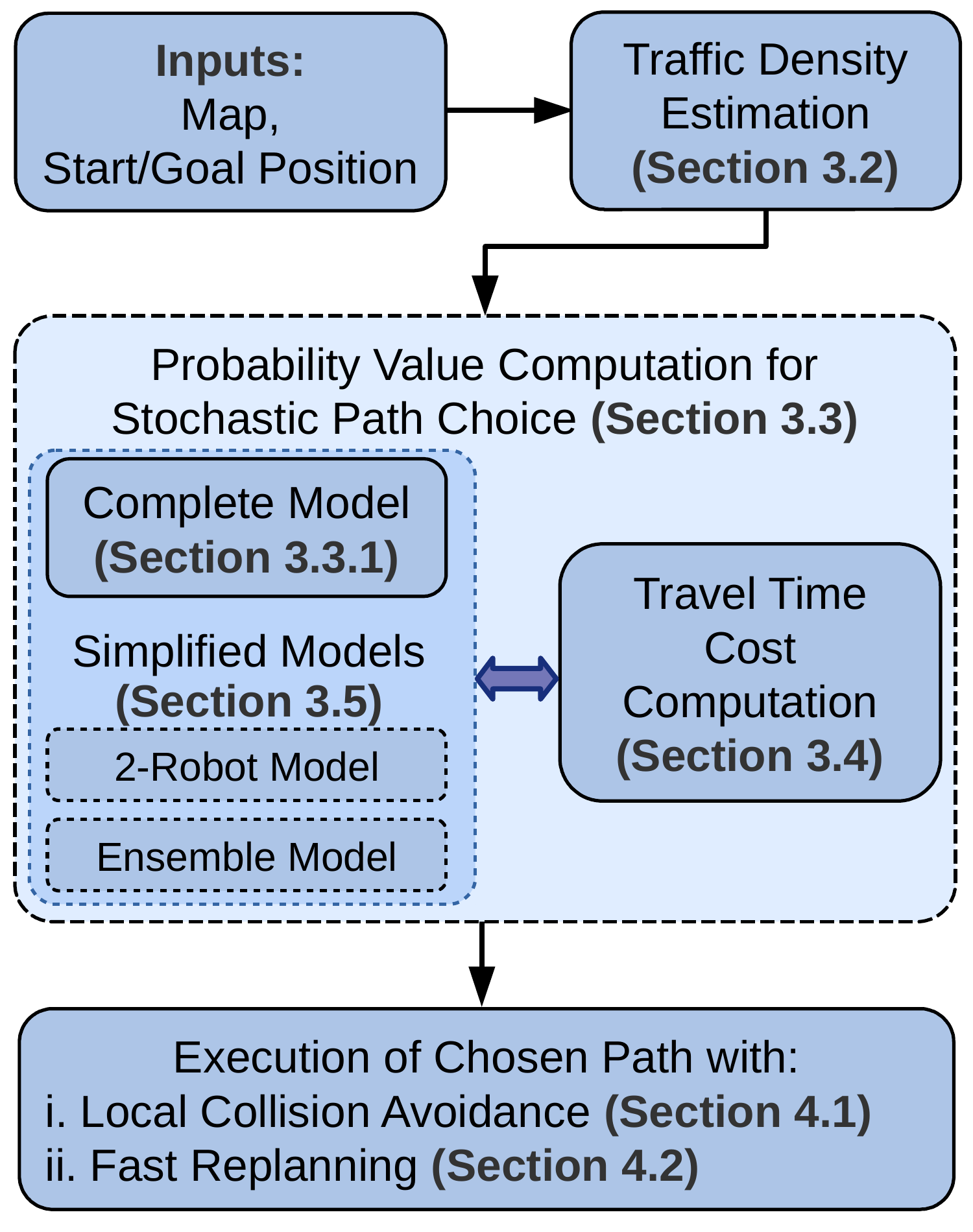}
  \end{center}
  \caption{\revisedXW{Workflow of the proposed coordination-free planning. Each robot follows this sequence of algorithms for computing and executing its own path without any inter-robot coordination or communication.}} \label{pic_pipeline}
  \vspace{-2em}
\end{wrapfigure}
\noindent
in the evaluation of the cost of path choices (Section~\ref{travel_cost_estimation}), which in turn is used in the computation of the probability values through an optimization process (Section~\ref{sec:probability-computation}). We also provide multiple appromixations and simplifications to the optimization problem for fast and efficient computation (Section~\ref{sec:simplified-model}).
% is incorporated within this optimization problem itself through a model for travel cost estimation (described in Section~\ref{travel_cost_estimation})
A robot stochastically chooses 
a path using the computed probability values.
We call a chosen path the \emph{reference path}
of the robot, and once chosen, a robot commits to it.
Section~\ref{sec:implementation} describes the controller used by each non-holonomic robot to follow the reference path while locally avoiding collisions with other agents in its immediate neighborhood.
% (\emph{i.e.}, which can be sensed using onboard sensors).
This includes a prediction of the future position of agents in the immediate neighborhood (Section~\ref{sec:probability-prediction}) and fast re-planning of path to avoid collision (Section~\ref{sec:cost-heuristic}).
Section~\ref{sec:results} provides simulation results and results from real-robot experiments.
%\todo{figure. paper outline.}
% 

The overall algorithm that each robot follows in computing and executing their respective paths is shown in the pipeline diagram of Figure~\ref{pic_pipeline}. The relevnat section numbers containing the details of each of the algorithmic components are also shown in the figure.
% \endgroup

% \begin{tabular}{c|c}
%     A &
% % \begin{wrapfigure}{r}{0.5\textwidth}
%   \begin{center}
%     \includegraphics[width=0.48\textwidth]{pipeline_new.pdf}
%   \end{center}
%   % \caption{Birds}
% % \end{wrapfigure}
% \end{tabular}

%\vspace{3em}

% ==========================================================

\section{Preliminaries -- Topological Planning and Rationalized Discretization in Spatio-Temporal Domain}% with Uncertain Pedestrians}
\label{sec:top-planning}

\changedSB{In this section we provide brief background on \emph{topological path planning} that allows the computation of shortest path in different topological classes using a graph search-based approach.
For further details on topological path planning the reader can refer to the author's prior work~\citep{planning:AURO:12,Homotopy-Planning:journal:18,Wang:Cable-Controlled:2018,DARS:14:HRI}. %,cable:separation:IJRR:14,Kim:IROS13,
The type of topological classes that we consider in particular is the $\mathbb{Z}_2$ homology class~\citep{Persistence-Plnning:TRO:15,ICRA:20:path:deconfliction}, and we describe the path planning in spatio-temporal domain in order to account for dynamic agents during replanning.}

\subsection{Background: $\mathbb{Z}_2$ Homology and $H_2$-signature}

% We start by providing a brief background on $\mathbb{Z}_2$ homology and homology invariants %\emph{homology signature} ($H$-signature) 
% in a planar domain. For further details on homology and $\mathbb{Z}_2$ homology, the reader can refer to the author's prior works~\citep{Persistence-Plnning:TRO:15,ICRA:20:path:deconfliction}.

Two trajectories connecting the same start and goal points on a planar domain are said to be in the same homology class if the closed loop formed by the two trajectories %is \emph{null homologous}, i.e., it 
forms the oriented boundary of a two-dimensional obstacle-free region~\citep{planning:AURO:12}.
The homology class of a loop can be quantified by winding numbers around the connected components of obstacles.
%
%In order to obtain a quantitative measure for identifying and representing homology classes, 
%one needs to construct a {\emph{homotopy invariant}} -- a function on the space of curves that uniquely identifies a curve's homology class.
%
In order to prevent the separate counting of the homology classes that loop around an obstacle multiple times, one can compute the homology in the ``mod $2$'' coefficient~\citep{Persistence-Plnning:TRO:15,ICRA:20:path:deconfliction}.
%where every odd winding number is identified with $1$ and every even winding number is identified with $0$ (Figure~\ref{pic:z2homology}), and is referred to as $\mathbb{Z}_2$-homology.
\changedF{Doing so identifies all the even winding numbers to $0$ and all the odd winding numbers to $1$. %, thus preventing
This prevents
the creation of separate homology classes for loops that wind around obstacles multiple times (Figure~\ref{pic:z2homology}) and we refer to this homology as $\mathbb{Z}_2$-homology.}

In order to quantitatively identify and represent $\mathbb{Z}_2$-homology class of trajectories \changedF{in a planar domain, $\mathcal{C} = \mathbb{R}^2 - \mathcal{O}$ (where $\mathcal{O} \subset \mathbb{R}^2$ is the obstacle set)}, 
% one needs to 
we construct a {\emph{homology invariant}} called \emph{$H_2$-signature}, %as follows:
\changedF{which is a function on the space of curves in $\mathcal{C}$ that uniquely identifies a curve's $\mathbb{Z}_2$-homology class. This \changedF{computation and the associated} construction (see \citep{cable:separation:IJRR:14,Persistence-Plnning:TRO:15} for more details) can be summarized as follows:}
In an environment with $o$ connected components of obstacles
% , the $\mathbb{Z}_2$-\emph{homology invariant} of a curve can be computed as follows: We 
we place a \emph{representative point}, $\zeta_i$, on the $i^\text{th}$ connected component and construct non-intersecting rays, $\{r_i\}_{i=1,2\cdots,o}$, emanating from the representative points.
\changedF{The $H$-signature of a curve, $\tau$, is then given by a vector of integers, $H(\tau) = [h_1, h_2,\cdots,h_o]\in\mathbb{Z}^o$, where $h_i$ is the winding number around the $i^{\text{th}}$ obstacle and can be computed by the number of times the curve intersects the ray emanating from $\zeta_i$ (in counting the number of intersections, crossing from left to right is considered positive and that from right to left is considered negative).
Subsequently, the $H_2$-signature of $\tau$ is defined as $H_2(\tau) = [h_1, h_2,\cdots,h_o] \mod 2 \in\mathbb{Z}_2^o$, in which the $i^{\text{th}}$ element assumes value in $\mathbb{Z}_2 = \mathbb{Z} / 2 \mathbb{Z} = \{0,1\}$ and gives the parity of the number of times the curve intersects the ray emanating from $\zeta_i$.
This computation is further illustrated in Figure~\ref{pic:z2homology}.
}

\begin{figure}
	\centering
	\includegraphics[width=0.95\columnwidth, trim=0 0 0 0, clip=true]{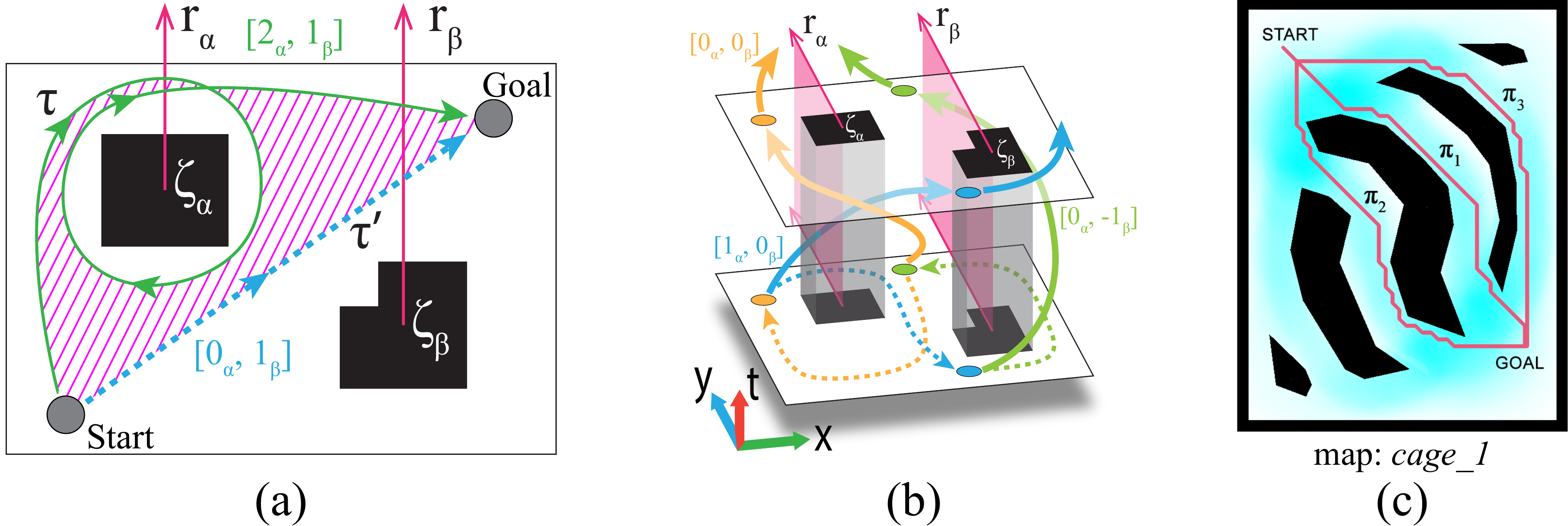}
	\caption{Topological Planning in $H_2$-augmented graph:
		\textbf{(a)} Homology classes of trajectories: $\tau$ and $\tau'$ are in different classes in regular homology \changedF{($H(\tau) = [2,1], H(\tau') = [0,1]$)}, but map to the same class in $\mathbb{Z}_2$-coefficient homology \changedF{($H_2(\tau) = H_2(\tau') = [0,1]$), where the elements of the vector are reduced `$\!\!\mod 2$'}. As a consequence, in computing paths in different homology classes, searching in the $H_2$-augmented graph does not return paths that loop around obstacles multiple times.
		\textbf{(b)} $H_2$-signature in a spatio-temporal configuration space.
		\textbf{(c)} Three paths in different $\mathbb{Z}_2$-homology classes (colored in pink) from the start to the goal in map ``\emph{cage\_1}'' computed using search in the $H_2$-augmented graph $\mathcal{G}_{H_2}$. Their \emph{base costs} (travel time cost computed by the A* search) are $\mathscr{C}_B(\pi_1)\!=\!8.803$, $\mathscr{C}_B(\pi_2)\!=\!9.803$, $\mathscr{C}_B(\pi_3)\!=\!10.603$. Cyan shade shows the estimated traffic density, $\rho$ \changedF{(described in Section~\ref{sec:traffic-density})}.
	}
	\subfiglabel{pic:z2homology}{(a)}
	\subfiglabel{pic_Haugmentedspace}{(b)}
	\subfiglabel{pic_classes}{(c)}
	%\label{x}
%	\vspace{-3em}
\end{figure}

\subsection{\changedB{Topological} Planning in $H_2$-Augmented Graph In Spatio-Temporal Domain}
\label{h2-aug-graph}
%$H$-signature is homology invariant of trajectories \textemdash two trajectories~connecting same start and end points have the same $H$-signatures iff they are in the same homology class. We use representative points for obstacles, $\zeta_i$, and the non-intersecting rays $r_i$ emanating from the representative points for constructing $H$-signature, shown in Figure \ref{pic_Haugmentedspace}. $H$-signature is a vector of integers, the $i^{\text{th}}$ element of which has the simple interpretation of counting the number of times the curve, $\tau$, intersects the ray emanating from $\zeta_i$ (cf., “winding numbers”). Crossing from left to right is positive and that from right to left is negative.
\changedB{The configuration space of each of the robots in our case is a discrete representation of the spatio-temporal domain. 
For a single robot, it is represented by the graph, $\mathcal{G} = (\mathcal{V},\mathcal{E})$, so that a vertex $\mathbf{v}\in \mathcal{V}$ is represented by $\mathbf{v} = (x,y,t)$, and edges connect neighboring vertices.
In order to keep track of the $\mathbb{Z}_2$ homology invariants, we define an $H_2$-augmented graph \changedF{(see \citep{planning:AURO:12,cable:separation:IJRR:14,Persistence-Plnning:TRO:15} for detailed construction)}, $\mathcal{G}_{H_2} = (\mathcal{V}_{H_2},\mathcal{E}_{H_2})$, based on the graph $\mathcal{G}$, %a graph in the aforesaid spatio-temporal configuration space, $\mathcal{G} = (\mathcal{V},\mathcal{E})$, 
such that a vertex in it %$\mathcal{G}_{H_2}$ 
is represented as $(\mathbf{v},\mathfrak{H})\in \mathcal{V}_{H_2}$, which contains the additional information of the $H_2$-signature, $\mathfrak{H}$, of the trajectory leading from a start vertex $\mathbf{v}_s \in \mathcal{V}$ up to the vertex $\mathbf{v}$. \changedSB{An edge connecting} vertex $(\mathbf{v},\mathfrak{H})$ to vertex $(\mathbf{v}',\mathfrak{H}')$ 
exists if 
% means that 
$(\mathbf{v},\mathbf{v}')\in \mathcal{E}$, and 
% $\mathfrak{H}' = \mathfrak{H} + H_2(\overline{\mathbf{v}\mthbf{v}'})$.
% the $H_2$-signature, 
$\mathfrak{H}'$
is the \changedB{sum} of $\mathfrak{H}$ and the $H_2$-signature of \changedB{the edge connecting} $\mathbf{v}$ to $\mathbf{v}'$.

In the spatio-temporal setup the \emph{rays} 
\changedF{in Figure \ref{pic:z2homology}}
emanating from obstacles
are \emph{extruded} in the temporal direction to construct half-planes (Figure~\ref{pic_Haugmentedspace}) and the intersections of trajectories are counted with these half-planes for computing the $H_2$-signature.
Thus in this case the $H_2$-signature computation for a path in the X-Y-T space requires counting the number of intersections of the path with each of those half-planes.
% Other details about the definition of the $H_2$-Augmented Graph can be referred to the authors' prior similar work \citep{wang2018topological}.

}
\changedB{Formally, given the graph, $\mathcal{G}$ and a start vertex $\mathbf{v}_s\in \mathcal{V}$, the vertex set, edge set and the cost function of the $H_2$-augmented graph can be described as:
%	~(see \citep{planning:AURO:12} for a similar construction):
}
\begin{enumerate}%[leftmargin=-50em]
	\item $(\mathbf{v},\mathbf{0}) \in \mathcal{V}_{H_2}$ is the start vertex, and the vertex set is given by:~
	%\vspace{-\baselineskip}
	\[\mathcal{V}_{H_2}=\left\{
	(\mathbf{v},\mathfrak{H})
	\left|
	\begin{aligned}
	&\mathbf{v}\in\mathcal{V},\text{ and}, \mathfrak{H} = H_2(\widetilde{\mathbf{v}_s\mathbf{v}})\\
	&\text{ for some trajectory } \widetilde{\mathbf{v}_s\mathbf{v}}\\
	&\text{from the start vertex }\mathbf{v}_s\text{ to }\mathbf{v} %\\
%	&\text{vertex }\mathbf{v}
	\end{aligned}
	\right.
	\right\}
	\]
	\item An edge $\{(\mathbf{v},\mathfrak{H})\to (\mathbf{v}',\mathfrak{H}')\}$ is in $\mathcal{E}_{H_2}$ 
%	for $(\mathbf{v,\mathfrak{H}}), (\mathbf{v',\mathfrak{H}'}) \in \mathcal{V}_{H_2}$, 
	iff $(\mathbf{v}\to\mathbf{v}')\in\mathcal{E}$ and \changedB{$\mathfrak{H}'= \left( \mathfrak{H}+ H_2(\mathbf{v}\to\mathbf{v}') \right) \!\!\mod 2$, where, $H_2(\mathbf{v}\to\mathbf{v}')$ is the $H_2$-signature of the directed segment representing the edge $(\mathbf{v}\to\mathbf{v}')$, and} ``$+$" is vector addition.
	\item The cost associated with an edge, \changedB{$C_{\mathcal{G}_{H_2}}\left( \{(\mathbf{v},\mathfrak{H})\to (\mathbf{v}',\mathfrak{H}') \} \right)$ is the same as the cost of the projected edge in, $\mathcal{G}$, \emph{i.e.}, $C_{\mathcal{G}}\left( \{\mathbf{v}\to\mathbf{v}'\} \right)$. The cost function is described in more details in Section~\ref{sec:cost-heuristic-initial-plan}.}
\end{enumerate}
\changedB{
Searching in this $H_2$-augmented graph from the start vertex $(\mathbf{v}_s,\mathbf{0})$ using A* search~\cite{Choset_2005}, paths to vertices of the form $(\mathbf{v}_g, *)$ give paths in different $\mathbb{Z}_2$ homology classes connecting $\mathbf{v}_s$ and $\mathbf{v}_g$ (Figure~\ref{pic_classes}).}
The vertices are generated on-the-fly and as required during the execution of \textit{A*} search on the graph.
\changedC{The paths obtained in the different homology classes are in ascending order of the cost.% (probability-weighted travel time -- described in next section) of the paths.
}

%\vspace{0.8em}
\subsection{\changedC{Rationalized Discretization of the Spatio-Temporal Configuration Space for Constructing $\mathcal{G}$}} \label{sec:discretization}

%We create a uniform grid representation of this 3-dimensional spatio-temporal domain with each grid cell of size $\delta x$, $\delta y$ and $\delta t$ along the X, Y and T coordinate axes directions respectively.
% \changedC{For constructing the $H_2$-augmented graph and performing A* searches,} 
\changedSB{In order to construct the configuration graph, $\mathcal{G}$, 
we discretize the spatio-temporal domain into a grid and place a vertex in every discrete cell in the grid.
The time axis is discretized uniformly with the time layers separated by $\delta t$ (Figure \ref{pic_successors_2}).
% 
% as the discrete representation of the spatio-temporal domain,}
% we discretize the time axis uniformly, with the time layers separated by $\delta t$ (Figure \ref{pic_timelayers}). However, in the spatial direction we choose a more flexible discretization approach in order to allow the robot have a speed that is direction-independent (isotropic).
% 
In the initial planning for the reference path we assume a constant robot speed, $V_{\text{max}}$, which is difficult to achieve with a uniform spatial discretization of $\delta r = V_{\text{max}} \delta t$ along X and Y directions (since diagonal edges \changedF{of an uniform square grid discreitization} are longer than the edges parallel to the coordinate axes). Instead, we discretize the spatial directions (both X and Y) using intervals of $\delta r' = \frac{\delta r}{4}$ and establish edges connecting a vertex $(x,y,t) \in \mathcal{V}$ with vertices of the form $(x \pm 3\delta r',y \pm 3\delta r', t+\delta t)$, $(x \pm 4\delta r', y, t+\delta t)$ and $(x, y \pm 4\delta r', t+\delta t)$ \changedF{(see Figure~\ref{pic_successors_1})}. In doing so, the spatial length of edges parallel to X or Y axes are $\delta r$, while the spatial length of the diagonal edges are $\sqrt{2}\frac{3}{4} \delta r = 1.06\cdots \delta r \approx \delta r$ (Figure~\ref{pic_successors_1}), thus allowing the implementation of almost-isotropic (direction-independent) velocity of $V_{\text{max}}$ along the edges.
We call this the ``\emph{Rationalized Discretization}''.
}

\begin{figure}
	\centering
	\includegraphics[width=0.9\columnwidth, trim=0 0 0 0, clip=true]{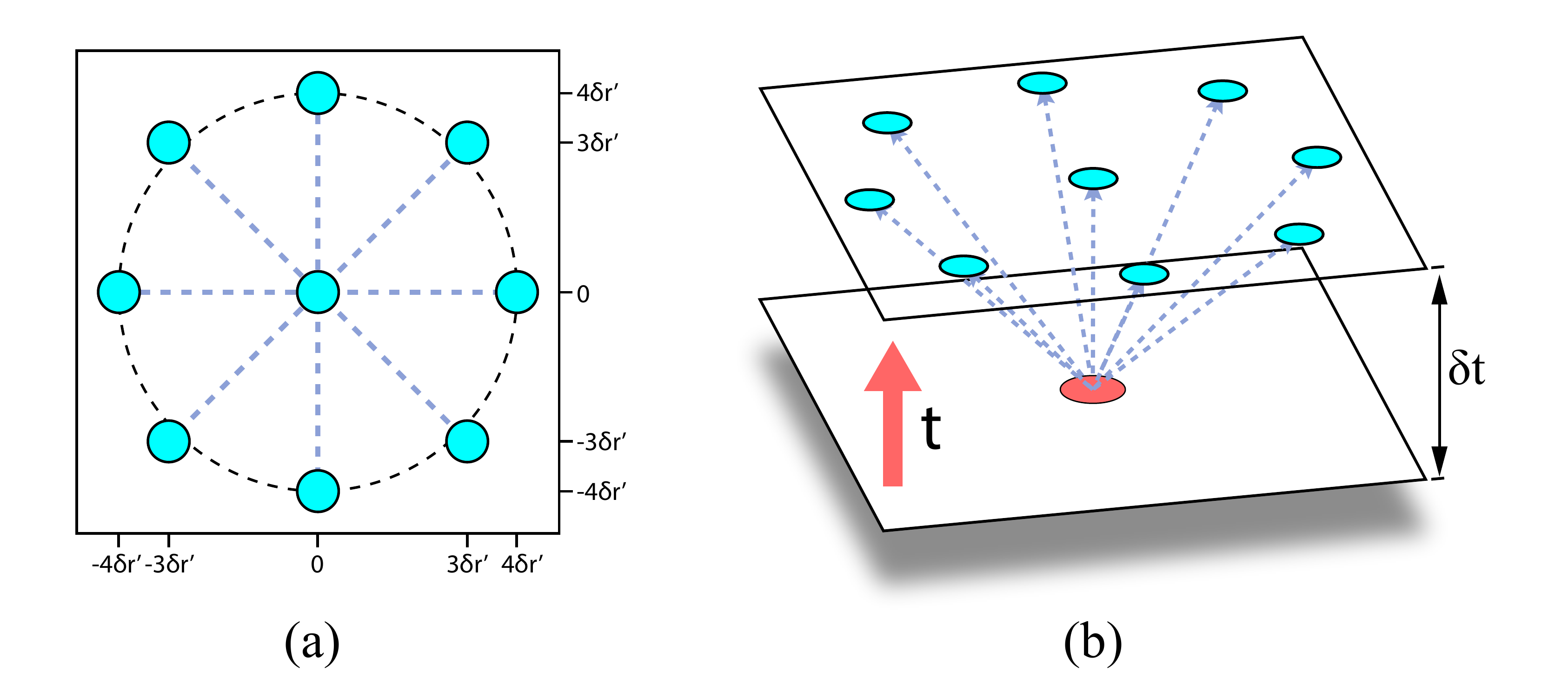}
	\caption{Successors (cyan) of a vertex $(x,y,t)\in \mathcal{V}$ (in red) using the \emph{rational discretization}. \changedF{Note how, although the successors with different spatial coordinates do not exactly fall on a circle of radius $V_{\text{max}} \delta t$, they do approximate the circle reasonably well due to the rationlized discretization.}}
\subfiglabel{pic_successors_1}{(a)}
\subfiglabel{pic_successors_2}{(b)}
\end{figure}

\changedB{
\subsection{Cost Function and Heuristic Function \changedF{for A* Search for Computing Paths in Different Topological Classes}} \label{sec:cost-heuristic-initial-plan}

For the topological planning of the paths in different $\mathbb{Z}_2$-homology classes for a particular robot,}
%the cost function used is a combination of the travel time and the Euclidean length. 
the cost function \changed{accounts for} the \changed{traversal time for the edges.}
%length of an edge that the robots have traveled in this 3-D space, 
%but also the probability of the edge being occupied by pedestrians.
%\changedB{As described earlier, let $P(x,y,t)$ be the current estimate of the probability that a discrete cell at $(x,y,t)$ is occupied by a pedestrian (details on updating that distribution is given in Section~\ref{sec:prediction}).} %The cost function on $\mathcal{G}$ is described as follows:}
%
%\changed{
% We define the \emph{effective time} of a small segment in the X-Y-T configuration space with projections $dx$, $dy$ and $dt$ along the coordinate axes as
% %\begin{equation}
% $d\tau = \sqrt{\epsilon \left( dx^2 + dy^2 \right) + dt^2 }$
% %\end{equation}
% with a small positive $\epsilon$ so as to ensure that $d\tau$ mostly represents the time span of the segment, but without allowing any zero-length edge in the graph.
% If \changedB{a} robot moves at uniform speed of $V_\text{max}$, we get $d\tau = \sqrt{\epsilon V_\text{max}^2+1} dt$.
\changedC{%To correct the error brought by rationalized disretization, 
The cost of an edge connecting vertices $(x,y,t, \mathfrak{h})$ and $(x',y', t+ \delta t, \mathfrak{h}')$ in $\mathcal{E}_{H_2}$ is thus chosen to be $\sqrt{ (x'-x)^2 + (y'-y)^2}/V_\text{max}$.}
\changedC{From a vertex $(x,y,t,\mathfrak{h})$ it will take at least $T_g = \frac{\sqrt{(x_g-x)^2 + (y_g-y)^2}}{V_\text{max}}$ time to reach a goal vertex of the form $(x_g, y_g, *, *)$ using any $\mathbb{Z}_2$ homology class, which is used as the heuristic function for A* search in $\mathcal{G}_{H_2}$.}
\changedF{An A* search in this $H_2$-augmented graph is used to find paths to vertices of the form $(x_g, y_g, *, *)$, which returns paths in increasing order of cost, and the cost of a path, $\pi$, is denoted by $\mathscr{C}_B(\pi)$, which is referred to as the \emph{base travel cost}.}

% =========================================================

\section{\changedSB{Stochastic Topological Path Assignment} for \changedSB{Coordination-free} Multi-robot System}
\label{sec:path-combination-evaluation}

\changedF{In this section we describe the algorithm that a robot uses for computing the probability values with which it stochastically assigns itself to one of the topologically distinct paths that it has computed. The main consideration in designing the algorithm is that the robots are not allowed to communicate or coordinate among themselves in either the computation of the probability values or in choosing their own path assignments. A robot does not share its location or its intent/choice with other robots. In the next section we start with a brief description of the problem setup and introduce some terminologies.}

\subsection{\changedF{The Environment and a Robot's Mental Model of the Environment}}

\begin{SCfigure}[2][h]
\centering 
\includegraphics[width=0.6\columnwidth]{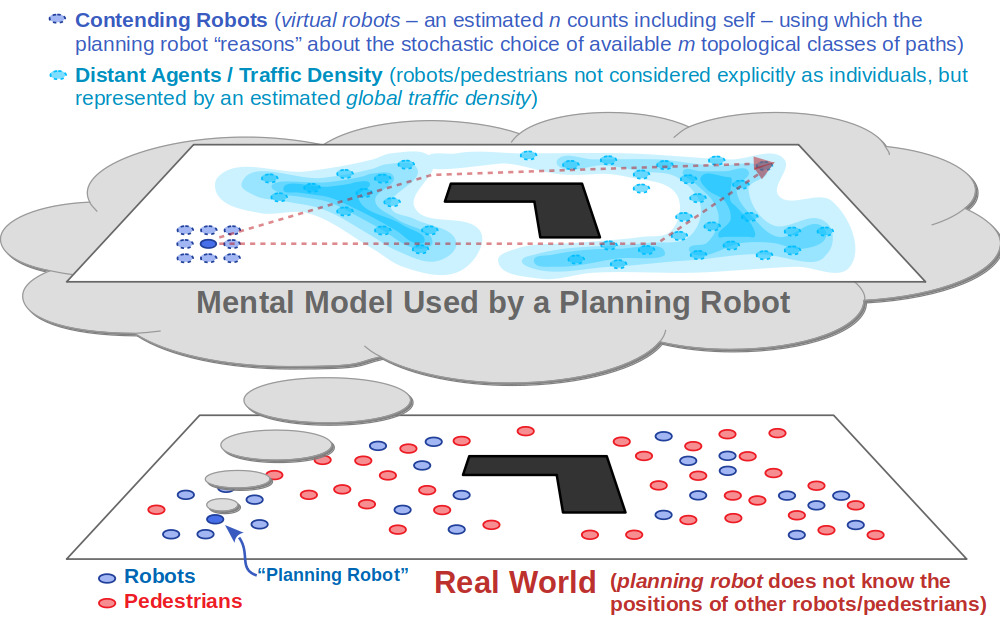}
% \vspace{-1.5em}
% \vspace{-0.8em}
\caption{\small\changedF{The mental model of a robot computing its own path (a \emph{planning robot}): In absence of the knowledge of the positions of other robots or pedestrians, a robot reasons about its stochastic choice from the available topological classes of paths using a team of virtual robots with similar start/goal (the ``\emph{contending robots''}). All other agents (robots or pedestrians -- referred to as ``\emph{distant agents}'') are accounted by an estimated traffic density map, $\rho$.}}
% \vspace{-1em}
\label{mental_model}
\end{SCfigure}

\changedB{We consider a planar \changedSB{indoor or urban} environment with multiple robots \changedSB{navigating} from one location to another \changedSB{within} the environment while trying to avoid global traffic congestion.
\changedSB{While the robots are \emph{rational} agents (their actions are determined by an algorithm that altruistically tries to reduce global traffic congestion), there also exist \emph{non-rational agents} (pedestrians) that do not attempt to reduce congestion in their path planning.}
% 
%\changedC{The robots' start locations are close to each other; likewise, their goal locations are near to each other.} 
Robots need to maintain a minimum safe distance from other robots as well as pedestrians. Because of that,}
%Suppose there are $n$ robots starting in different locations at the same time and having different goals. Due to the safety distance between robots and pedestrains, 
if a \changedSB{passage or route in the environment} becomes too crowded, some \changedB{may} have to slow down or wait %or even \changedB{move}
for the congestion to reduce.
% 
% to take a detour to reach their goals.
\changedB{Due \changedSB{to lack of inter-robot coordination, \changedF{a lack of knowledge of the current global traffic state/distribution in the environment,} and} the unpredictable nature of the pedestrians, it is virtually impossible to predict such congestion ahead of time. As a consequence,}
the overall travel time of all robots could increase \changedB{significantly}. 
\changedB{The key insight in addressing this problem is to distribute the robots across} different \changedB{routes in the environment} \changedSB{so as to} lower the \changedB{probability of congestion}. \changedB{Assigning the robots paths in different} topological ($\mathbb{Z}_2$ homology) \changedB{classes in the environment} can help achieving \changedB{that}. 
% \changedC{There is no inter-robot communication allowed in this whole process, so each robot has to choose its own way without knowing others' choices.}
% 

\subsubsection{Types of Agents and the ``Planning Robot''}

\changedSB{In the environment we assume that there are two types of agents -- \textbf{i.} \emph{non-rational agents}, also referred to as \emph{pedestrians}, that always choose the shortest path without consideration for global congestion reduction, and \textbf{ii.} \emph{rational agents}, also referred to as \emph{robots}, that stochastically chooses one of the multiple topologically distinct paths available to it with an aim of reducing global congestion.}
\changedSB{%Each robot needs to choose a path in a topological class stochastically in order to altruistically reduce the global traffic congestion.
Without coordination between the robots, all computation of the paths and stochastic path selection for a particular robot happen onboard the robot itself in a decentralized manner, without communication with other robots. In the following sections, we describe the computations made by a particular robot (\changedF{\emph{i.e.},} from the perspective of that individual robot), referred to as ``\emph{the planning robot}''. It is to be noted that all robots are \emph{planning robots} in their own rights, and the same algorithm is used by each robot for its individual computation.}

\subsubsection{\changedSB{Decoupling the Problem to Reason About Global Agent Distribution and Local Path Selection}}

\changedSB{%Each robot needs to choose a path in a topological class stochastically.
A \emph{planning robot}
needs to not only reason about other robots that may have similar start and goal locations as itself, but also the robots that may have different start/goal locations as well as the pedestrians. Without knowing the location, intent, or choices of other agents in the environment, a planning robot
% (henceforth referred to as ``\emph{the planning robot}'') 
decouples this complex problem into two parts:
\textbf{i.} \emph{Estimation of global traffic density, $\rho$, in the environment based on a prior belief of agent trajectories} (Section~\ref{sec:traffic-density}), and,
\textbf{i.} \emph{Probabilistic choice of topologically distinct paths for avoiding congestion} (Sections~\ref{sec:probability-computation}--\ref{sec:simplified-model}).

% \subsubsection{\changedF{A Robot's Mental Model of the Environment}}

\changedF{A planning robot does not know the location of other agents (pedestrians or other robots) in the environment.
In order to reason about the available topological classes connecting its own start and goal locations, it considers a group of virtual robots with similar start \& goal locations. These robots are referred to as \emph{contending robots}, and having similar start/goal locations have the same topological classes of paths as the planning robot and thus enables reasoning about the stochastic choice of a class (Figure~\ref{mental_model}).
Other agents (referred to as \emph{distant agents}) in the environment are represented by an estimated traffic density map, which is used in evaluating the different topological classes of paths when making the stochastic choice.}

% \todo{define ``the planning robot'', rational agents (robots) and non-rational agents (pedestrians). Disk-shaped footprint.}

\subsection{Estimation of Traffic Density in an Environment} \label{sec:traffic-density}

\begin{figure}
	\begin{minipage}[c]{0.6\textwidth}
		\includegraphics[width=\columnwidth, trim=0 0 0 0, clip=true]{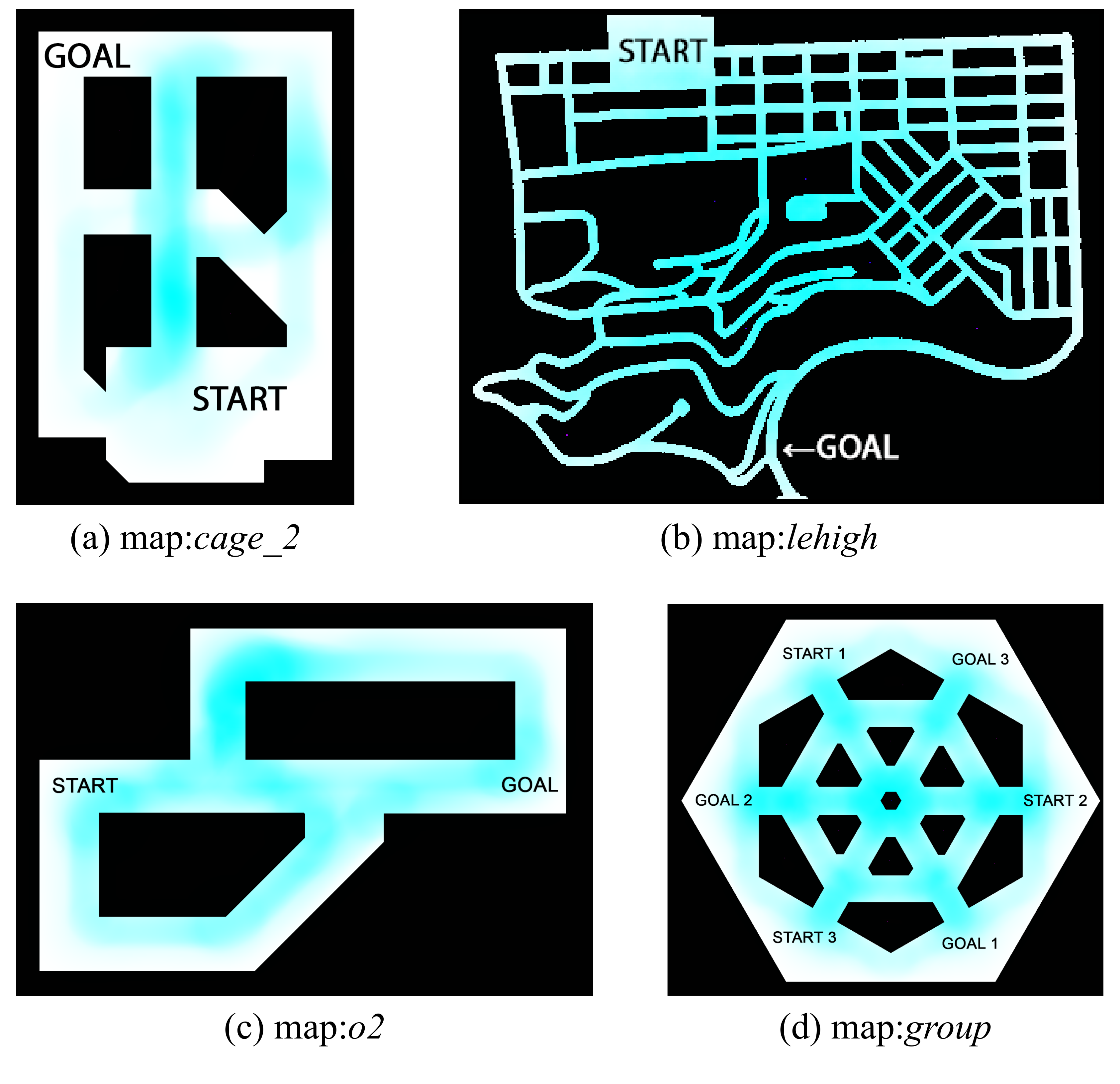}
	\end{minipage}\hfill
	\begin{minipage}[c]{0.37\textwidth}
		\caption{Maps used for simulations and experiments. Traffic density maps \changedF{estimated without a priori knowledge of any historic traffic data} are shown using shades of cyan -- darker cyan indicates potentially busier traffic regions.}
		\label{pic_density_map}
		\subfiglabel{pic_density_map_1}{(a)}
		\subfiglabel{pic_density_map_2}{(b)}
		\subfiglabel{pic_o2}{(c)}
		\subfiglabel{pic_group}{(d)}
	\end{minipage}
\end{figure}

This density estimation not only accounts for 
pedestrians, but also potentially accounts for robots
% rational agents (robots) 
that have start \& goal locations that are widely different from the planning robot.
% as well as for the non-rational agents (). 
We refer to such agents as \emph{distant agents}.
% The computation of this agent density is performed only once for a given environment by considering various potential start \& goal locations of other agents and the paths that they can tentatively choose. 
This density is used in formulating
% for two purposes: \textbf{a.} In 
the optimization problem for computing the probability values associated with the different topologically distinct path options available to the planning robot.
% , and, \textbf{b.} As a weight factor in the cost function used in A* search-based path planning.

% The traffic density, $\rho: V \rightarrow \mathbb{R}_{+}$, in a given environment is estimated a priori to account for non-rational agents and agents with different start and goal locations from the planning robot.
The distant agents' traffic density, $\rho: V \rightarrow \mathbb{R}_{+}$, is described by a real number associated with each cell (pixel) in the discrete representation of the environment.
While $\rho$ can be computed from historic traffic data, in absence of such data it can also be estimated from the structure of the environment. In order to do that, traffic is randomly generated with thousands of shortest paths connecting random starts to goals. Then each path is \emph{thickened} by Minkowski summing it with the pedestrians' disk-shaped footprint, and for each discrete cell in the map the number of such thickened paths that pass through it is counted. This distribution is then normalized to obtain the density function, $\rho$. Some examples of traffic density maps computed this way are shown in Figure \ref{pic_density_map}. This traffic density is an estimation only, related to the map's topology and geometry, and needs to be computed only once for a given environment.
}

\subsection{Probabilistic Choice of a Topological Class} \label{sec:probability-computation}

\changedSB{
In order to avoid congestion along the possible routes that the planning robot can choose to reach its goal, the robot \changedF{needs to} reason about the choices made by other robots that have similar start and goal locations as itself. To that end the robot can choose from multiple topological classes of paths representing distinct routes leading to its goal. \changedF{However,} without coordination between the robots, a probabilistic approach is taken in which the planning robot reasons about an \changedF{estimated} $n$ robots, including itself, with similar start and goal (refereed to as \emph{contending robots}) and chooses a topological class stochastically based on probability values computed to minimize the travel time for all the contending robots.

The planning robot starts by computing $m$ topologically distinct paths connecting its start to its goal location (Section~\ref{sec:top-planning}). The paths are 
% numbered from $1$ to $m$, and paths are 
referred by by their number/index from the set $S=\{1,2,\dots,m\}$.
The planning robot needs to choose one out of these $m$ classes stochastically. Suppose $P_j$ is the probability with which the robot chooses the $j$-th path.
In order to compute the path choice probabilities, $\{P_j\}_{j\in S}$, the robot reasons about the time of travel if all the contending robots choose the paths from $S$ according to the probabilities $\{P_j\}_{j\in S}$ (the contending robots being rational agents will use the same probability computation method themselves, thus arriving at the same probability values).

\subsubsectionB{Path Choice Probability Computation -- The Complete Formulation}

% In particular, the robot considers the scenario where $n$ robots, all starting at the same (or in close proximity of the) start location as itself, need to choose a path from $S$ to reach the same target location as itself.
%The planning robot assumes $(n-1)$ robot associates so that it reasons about $n$ rational agents (including itself) with the same or similar start \& goal locations (henceforth referred to as the \emph{contending robots}).
The value of $n$, while unknown, can be estimated based on the a priori knowledge of agent density in the environment or through local sensing.
% , although the exact value is not of critical importance since 
However, the planning robot uses the value $n$ only for the computation of its own path choice probability values $\{P_j\}_{j\in S}$, while in reality there is no real coordination between the planing robot and the other robots.}
% , which are used by its own self in determining which route it chooses stochastically.
% In the following discussions we refer to the planning robot and its presumed $(n-1)$ robot associates as \emph{contending robots} or \emph{contending agents}.
% 
%The robots search for $m$ \changedSB{topologically distinct paths (connecting their common/similar start and goal locations) which are numbered from $1$ to $m$, and the set of those paths is} denoted \changedSB{by $S=\{1,2,\dots,m\}$. Each robot needs to choose one out of these $m$ classes stochastically. Suppose $P_j$ is the probability with which a robot chooses the $j$-th path. In this section, we formulate the optimization problem for computing the path choice probabilities, $\{P_j\}_{j=1,\cdot,m}$, to be used by each of the $n$ coordination-free robots.}

\changedSB{We denote the set of contending robots as $R = \{1,2,\cdots,n\}$. Suppose the} $i$-th contending robot's choice of path is $\sigma_i\in S$, \changedSB{for} $i\in R$.
%All robots' choices are supposed to minimize the collective travel time. In each robot's processor, there runs a decision-making algorithm that aims to choose a path by solving an optimization problem with probabilities of paths as variables, say the probability of robot 1 choosing path $r_1$ being $P_{r_1}$, then randomly chooses one of the paths based upon the probabilities. Although different objectives could lead to different probability distribution and thus different choices, they all have a similar form: a travel cost expectation which includes the travel cost of all possible $n$-tuples that the robots could form, 
\changedSB{We define a \emph{joint path choice} made by the $n$ contending robots to be}
$\sigma=(\sigma_1,\sigma_2,\dots,\sigma_n)\in S^n$.
\changedSB{Given a joint path choice $\sigma\in S^n$, suppose $C(\sigma)$ is the estimated travel time cost of the entire group of contending robots (which is determined by the geometry of the $m$ paths, the prior estimated traffic density, $\rho$, along those paths, and the number of contending robots in each of those paths due to the joint path choice, $\sigma$ -- the computation of the cost $C(\sigma)$ is described in details in Section~\ref{travel_cost_estimation}).
Since the robots make their individual choices independently, the probability of making the joint path choice $\sigma$ is $\prod_{i=1}^n P_{\sigma_i}$.
We thus formulate the following optimization problem for minimizing the expected cost:}

\vspace{-1.2em}
%This is an $n$-order optimization problem and requires a non-linear programming to solve:
{\small \begin{eqnarray}
& & \quad\min_{P_1,P_2,\cdots,P_m} \sum_{\sigma\in S^n} C(\sigma) \prod_{i=1}^n P_{\sigma_i} %P_{r_1}P_{r_2}\cdots P_{r_n} 
\label{eq:optimization} \\[-0.5em]
\text{s.t.~~} & & \!\!\!\!\qquad \sum_{j=1}^m P_{j} = 1~, 
%\nonumber \\
%& & 
\qquad 0 \leq P_j \leq 1, ~\forall j\in S \nonumber
\end{eqnarray}}

\vspace{-1em}\noindent
%This is, in general, a nonlinear optimization problem 
where the summation in the objective function is over all possible joint path choices, and hence involves $m^n$ terms.

 \begin{lemma}
 The optimization problem \eqref{eq:optimization} is convex.
 \end{lemma}
 \begin{proof}
 The equality constraint is clearly affine and the inequality constraints are linear.
 Define the objective function $f(\{P_j\}_{j\in S}) = \sum_{\sigma\in{S}^n} C(\sigma) \prod_{i=1}^n P_{\sigma_i}$, which is homogeneous of degree $n$ in the probabilities.
 %Clearly, $f(\{\lambda P_j\}_{j\in S}) = \lambda^n f(\{P_j\}_{j\in S})$.
 %Thus for a $0\leq \lambda \leq 1$
 Thus, for two sets of probability values, $\{P^{(1)}_j\}_{j\in S}$ and $\{P^{(2)}_j\}_{j\in S}$, and with $0\leq \lambda\leq 1$,
 \begin{eqnarray}
 & & f(\{\lambda P^{(1)}_j\}_{j\in S}) + f(\{(1-\lambda) P^{(2)}_j\}_{j\in S}) \nonumber \\
 & = & 
 \lambda^n f(\{P^{(1)}_j\}_{j\in S}) + (1-\lambda)^n f(\{P^{(2)}_j\}_{j\in S}) \nonumber \\
 & \leq & \lambda f(\{P^{(1)}_j\}_{j\in S}) + (1-\lambda) f(\{P^{(2)}_j\}_{j\in S}) \nonumber 
 \end{eqnarray}
 Hence the objective function is convex.
% \qed
 \end{proof}

Even though this optimization problem is convex,
%However, 
the number of terms in the objective function grows exponentially (or factorially, upon some simplification) with $n$. %\todo{simplify the objective.} 
Hence, for all practical purposes, a direct solution to this optimization problem is not feasible since the evaluation of the objective function takes a lot of time.
% , despite the problem being convex. 
Hence in Section~\ref{sec:simplified-model} we will formulate two simplified optimization problems that are computationally more amenable.

\subsection{Travel Time Cost Computation} \label{travel_cost_estimation}

In this section, we describe the computation of the travel time cost function $C:S^n \rightarrow \mathbb{R}_{+}$ that estimates the travel time of the team of contending robots for a joint path choice, $\sigma$. 
%With the $i$-th robot choosing the path $\sigma_i$, we consider two types of travel time costs: 1) \emph{Average Travel Time}, which is the average of the travel times of all the robots, and, 2) \emph{Maximum Travel Time}, which is the maximum out of the travel times of all the robots.
%
% Given a joint path choice, $\sigma$, suppose $C_i(\sigma)$ is the estimated time taken by the $i$-th contending robot to travel its chosen path, $\sigma_i$. 
We first observe that the individual travel time of the $i$-th contending robot will not only depend on its own chosen path, $\sigma_i$, but also the choices made by the other contending robots since that will determine the level of congestion along the different parts of the path.
Define the set of contending robots that choose the path $j\in S$ as $R_j = \{i ~|~ \sigma_i = j\} \subseteq R$ and the number of those contending robots %that choose the path $j\in S$ 
as $N_j(\sigma) = \left| R_j(\sigma) \right|$.
Due to uniformity between the contending robots, if two contending robots, $i_1$ and $i_2$,
choose the same path (say, $j = \sigma_{i_1} = \sigma_{i_2}$),
then their estimated individual travel times will be the same.
% , $i_1$ and $i_2$, choose the same path, we have $C_{i_1}(\sigma) = C_{i_2}(\sigma)$.
Thus, for a given joint path choice $\sigma$, we define %for each path $j\in S$, 
the \emph{estimated travel time cost for the $j$-th path}
% taken by any contending robot to travel a path $j\in S$ 
as $D_j(\sigma)$.
% , which is the estimated travel time for any robot that chooses the $j$-th path given a joint path choice $\sigma$. %for any $i\in R_j$.
% 
We define two possible types of travel time cost functions for use in the objective function of \eqref{eq:optimization}:

\vspace{-0.3em}
\begin{itemize}
    \item[1.] \emph{Average Travel Time Cost:}
    {\small $C_\text{avg}(\sigma) \!=\!\displaystyle \frac{1}{n} \sum_{j=1}^m N_j(\sigma) D_j(\sigma)$},
    
    % \vspace{-0.5em}
    % {\small \[ C_\text{avg}(\sigma) 
    % % = \frac{1}{n} \sum_{i=1}^n C_i(\sigma) 
    % ~=~ \frac{1}{n} \sum_{j=1}^m N_j(\sigma) D_j(\sigma) \]}
    
    % \vspace{-1em}
    using which would try to minimize the average of the travel times of all the contending robots.
    \item[2.] \emph{Maximum Travel Time Cost:}
    {$C_\text{max}(\sigma) \!=\! \displaystyle \max_{j\in S} D_j(\sigma)$},
    
    % \vspace{-0.5em}
    % {\small \[ C_\text{max}(\sigma)=\max_{i\in \{1,2,\cdots,n\}} C_i(\sigma) ~=~ \max_{j\in S} D_j(\sigma)\]}
    
    % \vspace{-1em}
    using would try to minimize the maximum out of the travel times of all the contending robots.
\end{itemize}

% In the rest of this section we describe the computation of $D_j(\sigma)$.
% -- the estimated 
%time taken by a contending robot to travel along the path in the $j$-th class 
% \emph{travel time cost for the $j$-th path}
% given the joint path choice $\sigma$. 
% This estimated time 
\noindent
The estimated travel time cost for the $j$-th path, $D_j(\sigma)$, for a given joint path choice $\sigma$,
not only depends on the number of contending robots assigned to the path in the $j$-th class, but also the number of contending robots assigned to the other paths in $S$, since those paths can potentially have geometric overlaps with the $j$-th path (Figure~\ref{pic_penalty}).
As described earlier, we use A* search in the $H_2$-augmented graph (Section~\ref{h2-aug-graph}) to compute distinct paths in the $m$ different topological classes connecting the start and the goal location of the planning robot. Let's refer to these paths as $\{\pi_j\}_{j\in S}$.
% We also compute the 
The travel cost for the $j$-th path is then computed as

\vspace{-0.4em}
{\small \begin{equation} \label{eq:Dj}
    D_j(\sigma) = \mathscr{C}_B(\pi_j) + a\, Q\, \mathscr{C}_T (\pi_j) + b\, \sum_{j'\in S}\! N_{j'}(\sigma)~ \mathscr{C}_P(\pi_j,\pi_{j'})~~
\end{equation}}

\vspace{-1.8em}\noindent
where,

\begin{figure}
	\centering
	\includegraphics[width=0.75\columnwidth]{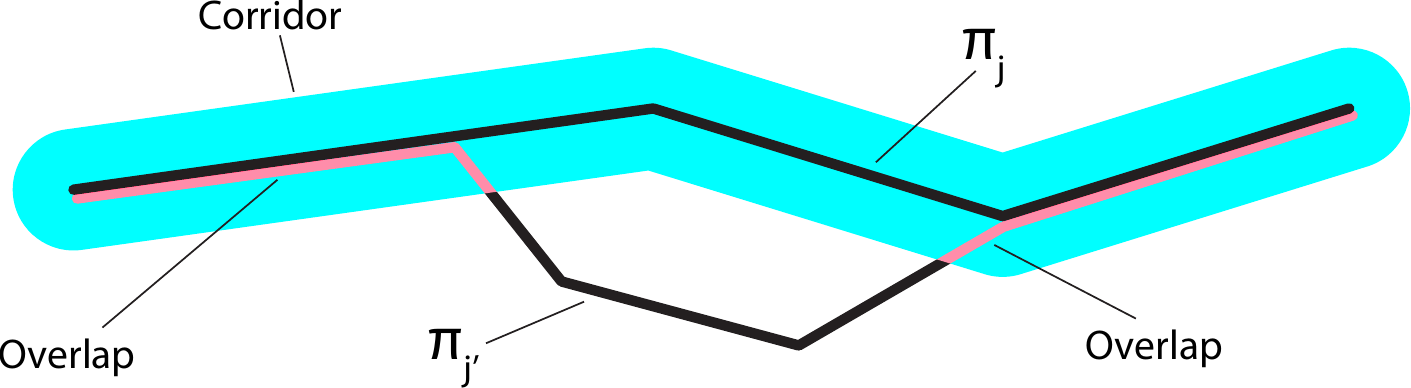}
%	\vspace{-0.4em}
\caption{In the proximity penalty computation, the overlap of two paths is determined by drawing a corridor around the path, $\pi_j$, then measuring the part of the path $\pi_{j'}$ that lies within that corridor.}
	\label{pic_penalty}
%	\vspace{-1.6em}
\end{figure}

% 
% 
% \setlength{\leftmargini}{-1em}
% \hspace{-2em} \noindent
\begin{itemize} %[leftmargin=*] %[wide, labelwidth=!, labelindent=0pt]
    \item[i.] $\mathscr{C}_B(\pi_j)$ 
    is the \emph{base travel cost}
    % refers to the time cost of the path 
    as computed by the search in the $H_2$-augmented graph (\emph{i.e.} estimated time taken to follow the path without consideration for any other agent in the environment). % and is referred to as the \emph{base travel cost}.
    \item[ii.] The second term computes the additional cost due to the a priori belief of traffic density, $\rho$, modeled to be  proportional\footnote{\label{foot:lin}Assuming a linearized model} to the net estimated traffic density along the path, $\mathscr{C}_T (\pi_j) = \sum_{s\in \pi_j} \rho(s)$\footnote{Here the summation over the path refers to the summation over the discrete cells 
    % (in its uniform pixel representation) 
    that constitute the path, with $\rho(s)$ being the density in cell $s$.} -- referred to as the \emph{traffic-weighted travel cost}, and scaled by the estimated number of distant agents, $Q$, in the environment. The proportionality constant, $a$, is determined experimentally (described in the next paragraph).
    %Section~\ref{sec:s-prop-const-computation}.
    \item[iii.] %$\mathscr{C}_O(\pi_j,\pi_{j'})$ 
    The last term computes the \emph{proximity penalty} or \emph{overlap cost} between pairs of paths due to multiple contending robots from different (or the same) paths %($\pi_j, \pi_{j'}, j'\neq j$) 
    creating congestion along the regions of $\pi_j$ where there is a geometric overlap with $\pi_{j'}$ (Figure~\ref{pic_penalty}). This includes \emph{self-overlap cost} (when $j'=j$) due to the multiplicity of contending robots following the $j$-th path. The cost is proportional${}^{\ref{foot:lin}}$ to the number of additional contending robots, $N_{j'}(\sigma)$, in the overlapping path and the \emph{amount of overlap}, $\mathscr{C}_P(\pi_j,\pi_{j'})$.
    The amount of overlap itself consists of two parts: $\mathscr{C}_P(\pi_j,\pi_{j'}) = \mathscr{C}_{P,B}(\pi_j,\pi_{j'}) + a\, Q\, \mathscr{C}_{P,T} (\pi_j,\pi_{j'})$, where the first part is the time cost of the part of the path $\pi_{j'}$ that overlaps with $\pi_j$ as is computed by the search in the $H_2$-augmnted graph (with the overlap being determined by the proximity between the points on the two paths -- Figure~\ref{pic_penalty}),
    % \todo{explain more here or the figure caption: We draw a corridor around the base path, then measure the length of the overlapping part.}), 
    and $\mathscr{C}_{O,T} (\pi_j,\pi_{j'})$ is the net estimated traffic density on the overlapping parts of $\pi_{j'}$.
\end{itemize}

\changedC{The values of $a$ and $b$ \changedSB{are determined experimentally by running multiple simulations} in a simple single-passage map (Figure \ref{pic_max_travel_time_relation}) with a varying number of distant agents (we choose pedestrians only) and a varying number of contending robots.
The width and the length of the passage in this map \changedSB{are chosen to be} similar to those in the \changedSB{maps used in experiments and simulations}. For more details on the implementation of the simulations, refer to Section~\ref{sec:simulation-experiment}.} 
%The results suggest that the cost of a path is positively associated with the number of robots and with the number of pedestrians .

\begin{figure}
	\centering
	\includegraphics[width=0.99\columnwidth, trim=0 0 0 0, clip=true]{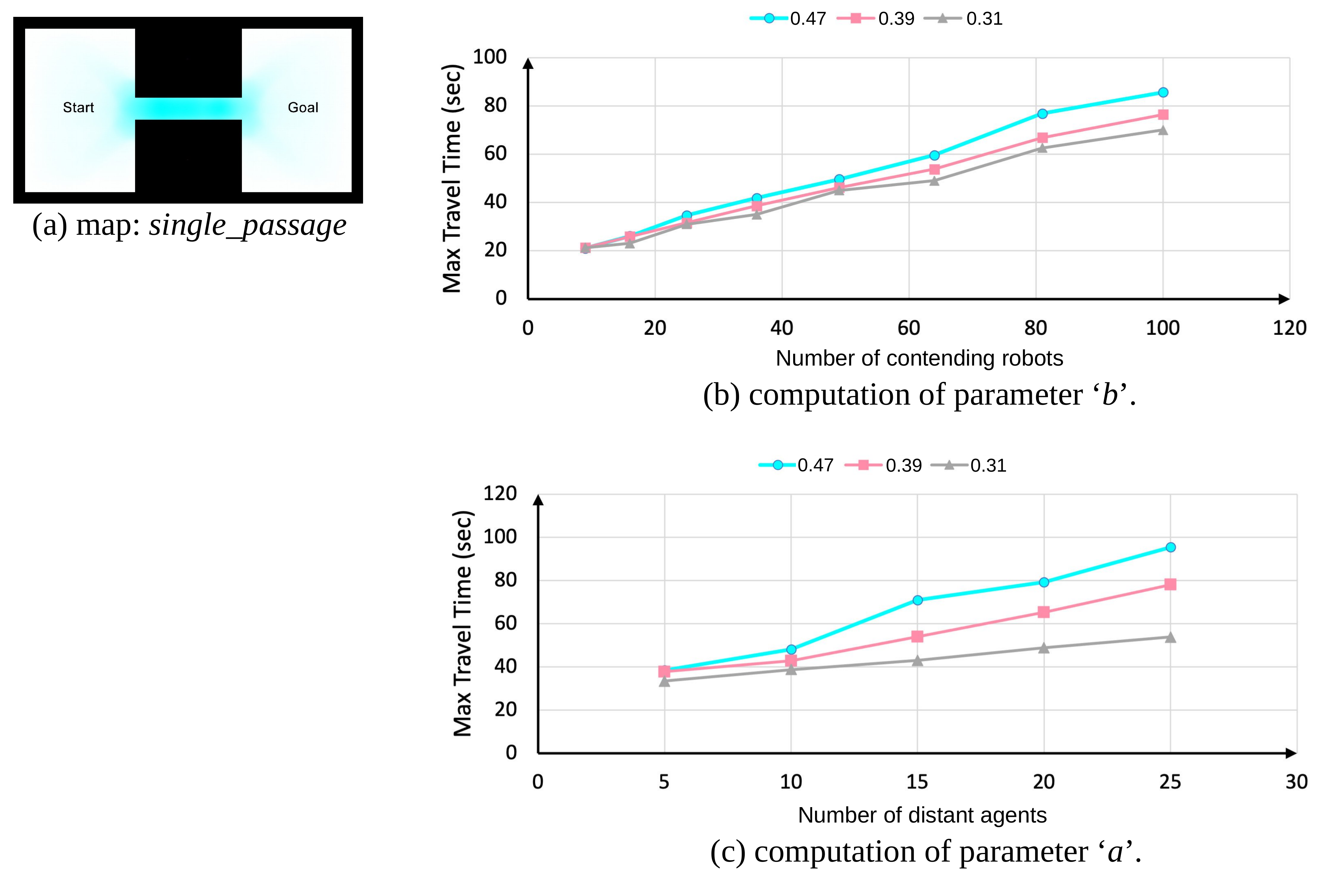}
	\caption{Experimental estimation of the proportionality constants $a$ and $b$ in \eqref{eq:Dj}.
		The curves in different colors in (b) and (c) are with different ratios of the robot's safety distance to the passage's width as that is different for different environments. 
		\textbf{(a)} \changedSB{The} single-passage map \changedSB{used to experimentally compute the proportionality constants $a$ and $b$}, with \changedSB{the} estimated traffic density, \changedSB{$\rho$} colored in cyan.
		\textbf{(b)} Max travel time increases with the number of contending robots, with the fixed number of pedestrians. The slope of this curve computes the constant $b$.
		\textbf{(c)} Max travel time increases with the number of pedestrians, with the number of robots fixed at 10. The slope of this curve computes the constant $a$.
		}
\label{pic_max_travel_time_relation}
\subfiglabel{fig:curve_map}{(a)}
\subfiglabel{fig:para}{(b)}
\subfiglabel{fig:para_2}{(c)} \vspace{-1em}
\end{figure}

% With function $F_k:\mathbb{S}^n\rightarrow\mathbb{N}$ where $F_k(\sigma)$ counts the number of robots choosing class $k$ in tuple $\sigma$, we can see that $\mathcal{C}_i(\sigma)=C_i(\sigma')$ if $F_k(\sigma)=F_k(\sigma'),\forall k\in\mathbb{S}$. In this way, less terms can be calculated in Equation \eqref{eq:optimization} to improve the performance, but it remains $n$-order.

\subsection{\changedSB{Simplified Formulations}} \label{sec:simplified-model}

\changedSB{The optimization problem in \eqref{eq:optimization} is referred to as the \emph{complete model} %, which requires nonlinear programming 
and has $O(m^n)$ terms in the objective function, making it extremely computationally expensive to solve with a large number of robots and available topological classes. We thus propose couple of approximations to simplify the optimization problem. These approximations rely on the fact that the number of \emph{contending robots}, $n$, that the planning robot assumes is an estimate and is purely for the purpose of computing its own path choice probabilities. In reality, there is no coordination or communication between nearby robots. We thus use extreme values of $n$ to simplify the models.}

% We call the model used in the aforementioned methods \emph{complete model}. Although the complete model can be simplified with combination of repetition instead of tuple, it still is an $n$-order optimization problem requiring non-linear programming solvers. The number of \changedC{rational agents} could be so large that the problem becomes extremely computationally expensive. To simplify the optimization problem, instead of the $n$-tuple model, some derivative ones can be applied.

\subsubsection{Two-robot Model}
In this model we assume that the number of contending robots the planning robot considers is $n=2$, so that the objective function in \eqref{eq:optimization} becomes quadratic, which can be solved efficiently using a quadratic program.
However, in computing 
% the \emph{travel time cost}, $C(\sigma)$, and 
the travel time cost for the $j$-th path, $D_j(\sigma)$, with $\sigma\in S^2$, we can still account for a number $n$ other than $2$ by simply replacing each of the $2$ robots with $\frac{n}{2}$ robots when computing the overlap costs. This is effectively done in \eqref{eq:Dj} by scaling \& redefining the robot counting function as $N_{j'}(\sigma) = \frac{n}{2} |R_{j'}(\sigma)|, ~\sigma\in S^2$.

% One can assume that the algorithm only coordinates a two-robot system instead of an $n$-robot one. The objectives are similar to all aforementioned ones, except replacing $n$-tuples with couples, $\sigma=(r_i,r_j)$:
% \begin{equation}
% \min_{P_1,P_2,\atop\cdots,P_m} \sum_{(r_i,r_j)\in\mathbb{S}^2} C_{r_ir_j} P_{r_i}P_{r_j}
% \end{equation}
% The whole problem can be reduced to $2$-order so that the objective function only needs a quadratic programming to solve for the probabilities. 
% Although the travel cost only involves two robots' choices, it is related to $n$, the actually total number of robots in this system. Derived from Equation \ref{eq:travel_cost}, the travel cost for the 2-robot model is
% \begin{equation*}
% \begin{split}
% C_{r_ir_j}=& \mathscr{C}_B(r_i)+aQ\mathscr{C}_T(r_i) \\
% & +nb\big(\mathscr{C}_B(r_i,r_j)+aQ\mathscr{C}_T(r_i,r_j)\big)
% \end{split}
% \end{equation*}
% When using quadratic programming, there is no difference between applying this for the total travel cost and for the own travel cost. Since the travel cost is constant in terms of probabilities, it can also be applied for the max travel cost method, where the optimization becomes
% \begin{equation}
% \min_{P_1,P_2,\atop\cdots,P_m} \sum_{(r_i,r_j)\in\mathbb{S}^2} \max(C_{r_ir_j},C_{r_jr_i}) P_{r_i}P_{r_j}
% \end{equation}

\subsubsection{Ensemble Model}

\changedSB{In this model the planning robot assumes a large number of contending robots so that the number of robots in the $j$-th class is approximately $n P_j$.
Considering the problem of minimizing the \emph{maximum travel time cost} (\emph{i.e.}, the maximum out of the travel times of all the contending robots) This allows us to reformulate the optimization problem as}

\vspace{-0.8em}
{\small \begin{eqnarray}
& & \min_{P_1,P_2,\cdots,P_m} \Big( \max_{j\in S} ~K_{j} (P_1,P_2,\dots,P_m)\Big)\label{eq:ensemble-optimization} \\[-0.4em]
\text{s.t.~~} & & %\!\!\!\!\!\!\!\!\qquad 
\sum_{j=1}^m P_{j} = 1~, 
%\nonumber \\
%& & 
\quad 0 \leq P_j \leq 1, ~\forall j\in S \nonumber
\end{eqnarray}}

\vspace{-1.3em}\noindent
\changedSB{where $K_j(P_1,P_2,\dots,P_m)$ is the estimated travel time cost for a robot assigned to the $j$-th path if the number of robots following the $l$-th path is $n P_l$ for all $l\in S$. The expression of $K_j$ is derived naturally from the definition of $D_j$ in \eqref{eq:Dj}:

\vspace{-0.7em}
{\small \begin{eqnarray}
% K_{j} (P_1,P_2,\dots,P_m) & = & \mathscr{C}_B(\pi_j) + a\, Q\, \mathscr{C}_T (\pi_j) + b\, \sum_{j'\in S}\! \mathscr{C}_P(\pi_j,\pi_{j'}) ~n P_{j'}\\
& & \!\!\!\!\!\!\!\!\!\!\!\! K_{j} (P_1,P_2,\dots,P_m) \nonumber \\[-0.5em]
& & = \mathscr{C}_B(\pi_j) + a\, Q\, \mathscr{C}_T (\pi_j) + b\, \sum_{j'\in S}\! \mathscr{C}_P(\pi_j,\pi_{j'}) ~n P_{j'} \nonumber \\[-0.5em]
& & \approx n b\, \sum_{j'\in S}\! \mathscr{C}_P(\pi_j,\pi_{j'}) P_{j'} \quad\text{(since $n$ is large)} \label{eq:Kj}
\end{eqnarray}}

\vspace{-0.2em}
\noindent
Note that $K_{j} (P_1,P_2,\dots,P_m)$ is affine in $\{P_j\}_{l\in S}$, and hence the objective function in \eqref{eq:ensemble-optimization} being $\max$ of affine functions, is convex \citep{boyd2004convex}. As a consequence, the optimization problem in \eqref{eq:ensemble-optimization} can be solved using efficient numerical methods and does not have an exponentially large number of terms as was the case in \eqref{eq:optimization}.
It is worth noting that there is no meaningful analogous ensemble model for minimization of the \emph{average travel time cost}
since the average of the affine functions, $\{K_{j}\}_{j\in S}$, would result in an affine objective function in \eqref{eq:ensemble-optimization}, which would result in a linear program, the solutions to which is always trivial with one of the probabilities in $\{P_j\}_{j\in S}$ being equal to $1$ and rest $0$.

% . This is because, if, instead of taking max of the affine functions, $\{K_j\}_{j\in S}$, in constructing the objective of \eqref{eq:ensemble-optimization}, had we taken the average, then the objective function would itself be affine in the $\{P_j\}_{j\in S}$. Although that would result in a linear program, the solutions to that would always be trivial with one of the probabilities in $\{P_j\}_{j\in S}$ being equal to $1$ and rest $0$.
}

\subsubsection{Path Choice Probability Values}

Given a planning robot's start and goal location in an environment, the path choice probabilities, \changedSB{$\{P_j\}_{j\in S}$,
depend on the choice of the model (\emph{complete model}, \emph{2-robot model} or \emph{ensemble model}) as well as the number of contending agents, $n$.
For the ensemble model, with the large $n$ assumption, it is clear from \eqref{eq:ensemble-optimization} and \eqref{eq:Kj} that the probability values are independent of the choice of $n$.
A comparison of the probability values computed using the different models and different $n$ is shown in Table \ref{tab_model_statistic}. The similarity among the values computed using the complete model and the $2$-robot model is apparent,
% . While the ensemble model gives significantly different path choice probability values, they 
while the values from the ensemble model
get closer to the $2$-robot model as the value of $n$ increases.
This allows us to choose a simplified model for fast computation of the probabilities in experiments and simulations with a large number of contending robots.}

% varies with the travel cost of each class and the number of \changedC{rational agents}, except the ensemble model. A comparison among all models and methods is shown in Table \ref{tab_model_statistic}. %This table suggests that all methods tend to distribute probabilities more evenly to classes as $n$ increases. 
% The similarity between the complete methods and the simplified methods is the reason we choose simplified models instead for experiments and simulations with a large number of rational agents.
%Table \ref{tab_model_statistic_pedestrian_quantity} shows the probability distribution changes as the number of pedestrians, $Q$, increases. It indicates that when more pedestrians showing up in the map, the probability is more evenly distributed to different classes.
\begin{table}
\begin{threeparttable} %\vspace{0.1in}   	
\caption{A comparison of the probability distribution over classes by different models on Map ``\emph{cage\_1}'' without non-rational agents. The paths of classes are shown in Figure \ref{pic_classes}.}
\label{tab_model_statistic}
\centering
\begin{tabular}{m{0.2\columnwidth}m{0.2\columnwidth}m{0.2\columnwidth}m{0.2\columnwidth}}
\toprule
%\begin{tabular}{@{}c@{}}Robot \#\\ $n$\end{tabular}
\multirow{2}{*}{\begin{tabular}{@{}c@{}}Contd. Rbt. \#\\ $n$\end{tabular}}& \multicolumn{3}{c}{Probability Distribution ($m=3$)\tnote{*}}\\
\cmidrule{2-4}
&Class 1 &Class 2 &Class 3\\
\hline %\midrule[1pt]
\multicolumn{4}{c}{\textbf{Complete Model:} Average Travel Time Cost} \\
\hline %\midrule[0.1pt]
5  & \cellcolor[HTML]{5CDFEA}0.64071 & \cellcolor[HTML]{A4EDF3}0.35917 & \cellcolor[HTML]{FFFFFF}0.00012 \\
10 & \cellcolor[HTML]{75E4ED}0.54207 & \cellcolor[HTML]{A6EEF4}0.35219 & \cellcolor[HTML]{E5FAFC}0.10574 \\
15 & \cellcolor[HTML]{82E6EF}0.49175 & \cellcolor[HTML]{A6EEF4}0.35085 & \cellcolor[HTML]{D7F7FA}0.15740 \\
20 & \textemdash\textsuperscript{**} & \textemdash\textsuperscript{**} & \textemdash\textsuperscript{**} \\
\hline %\midrule[1pt]
\multicolumn{4}{c}{\textbf{Complete Model:} Maximum Travel Time Cost}\\
\hline %\midrule[0.1pt]
5  & \cellcolor[HTML]{3ED9E6}0.76014 & \cellcolor[HTML]{C2F3F7}0.23985 & \cellcolor[HTML]{FFFFFF}0.00000 \\
10 & \cellcolor[HTML]{53DDE8}0.67781 & \cellcolor[HTML]{ADEFF5}0.32215 & \cellcolor[HTML]{FFFFFF}0.00004 \\
15 & \cellcolor[HTML]{74E4ED}0.54728 & \cellcolor[HTML]{9EECF3}0.38163 & \cellcolor[HTML]{EDFCFD}0.07109 \\
20 & \textemdash\tnote{**}& \textemdash\textsuperscript{**} & \textemdash\textsuperscript{**} \\
\hline %\midrule[1pt]
\multicolumn{4}{c}{\textbf{$2$-robot Model:} Average Travel Time Cost}\\
\hline %\midrule[0.1pt]
5  & \cellcolor[HTML]{4BDBE7}0.70690 & \cellcolor[HTML]{B5F1F6}0.29310 & \cellcolor[HTML]{FFFFFF}0.00000 \\
10 & \cellcolor[HTML]{73E3ED}0.55048 & \cellcolor[HTML]{A8EEF4}0.34206 & \cellcolor[HTML]{E4FAFC}0.10746 \\
15 & \cellcolor[HTML]{84E7EF}0.48449 & \cellcolor[HTML]{A9EEF4}0.33841 & \cellcolor[HTML]{D2F6F9}0.17710 \\
20 & \cellcolor[HTML]{8CE8F0}0.45323 & \cellcolor[HTML]{AAEEF4}0.33669 & \cellcolor[HTML]{CAF5F8}0.21009 \\
25 & \cellcolor[HTML]{91E9F1}0.43499 & \cellcolor[HTML]{AAEEF4}0.33568 & \cellcolor[HTML]{C5F4F8}0.22933 \\
30 & \cellcolor[HTML]{94EAF1}0.42304 & \cellcolor[HTML]{AAEEF4}0.33502 & \cellcolor[HTML]{C2F3F7}0.24194 \\
\hline %\midrule[1pt]
\multicolumn{4}{c}{\textbf{$2$-robot Model:} Maximum Travel Time Cost}\\
\hline %\midrule[0.1pt]
5  & \cellcolor[HTML]{5EDFEA}0.63448 & \cellcolor[HTML]{A2EDF3}0.36552 & \cellcolor[HTML]{FFFFFF}0.00000 \\
10 & \cellcolor[HTML]{88E8F0}0.46863 & \cellcolor[HTML]{A2EDF3}0.36812 & \cellcolor[HTML]{D6F7FA}0.16325 \\
15 & \cellcolor[HTML]{94EAF1}0.42046 & \cellcolor[HTML]{A4EDF3}0.35715 & \cellcolor[HTML]{C7F4F8}0.22239 \\
20 & \cellcolor[HTML]{99EBF2}0.40036 & \cellcolor[HTML]{A6EEF4}0.35227 & \cellcolor[HTML]{C0F3F7}0.24737 \\
25 & \cellcolor[HTML]{9CECF2}0.38933 & \cellcolor[HTML]{A6EEF4}0.34950 & \cellcolor[HTML]{BDF2F7}0.26117 \\
30 & \cellcolor[HTML]{9EECF2}0.38237 & \cellcolor[HTML]{A7EEF4}0.34772 & \cellcolor[HTML]{BBF2F6}0.26991 \\
\hline %\midrule[1pt]
\multicolumn{4}{c}{\textbf{Ensemble Model:} Maximum Travel Time Cost}\\
\hline %\midrule[0.1pt]
% 0  & \cellcolor[HTML]{A2EDF3}0.36570 & \cellcolor[HTML]{ABEFF4}0.33185 & \cellcolor[HTML]{B2F0F5}0.30245 \\
% 5  & \cellcolor[HTML]{A2EDF3}0.36570 & \cellcolor[HTML]{ABEFF4}0.33185 & \cellcolor[HTML]{B2F0F5}0.30245 \\
% 10 & \cellcolor[HTML]{A2EDF3}0.36570 & \cellcolor[HTML]{ABEFF4}0.33185 & \cellcolor[HTML]{B2F0F5}0.30245 \\
% 15 & \cellcolor[HTML]{A2EDF3}0.36570 & \cellcolor[HTML]{ABEFF4}0.33185 & \cellcolor[HTML]{B2F0F5}0.30245 \\
Any $n$ & \cellcolor[HTML]{A2EDF3}0.36570 & \cellcolor[HTML]{ABEFF4}0.33185 & \cellcolor[HTML]{B2F0F5}0.30245 \\
\bottomrule
\end{tabular}
\begin{tablenotes}
\item[*] All methods use $a=0.1625$, $b=0.04548$, and $Q=0$.
\item[**] Not computable due to the memory overflow in computation of $3^{20}$ terms in the objective function of the complete model.
\end{tablenotes}
%\vspace{-2.5em}
\end{threeparttable}
\end{table}

\section{Execution of Chosen Path with {Local Collision Avoidance}}
\label{sec:implementation}
\label{sec:simulation-experiment}

\changedF{In this section we describe the algorithm and controller used by a robot for following the stochastically chosen path in a topological class while avoiding immediate collisions based on local sensing.}

The overall algorithm for each robot \footnote{\changedF{In this section we referred to a \emph{planning robot} simply as a \emph{robot} without any risk of confusion.}} is as follows: 
A robot stochastically chooses a \emph{reference path} from the paths $\{\pi_j\}_{j\in S}$ that it computed using A* search in the $H_2$-augmented graph according to its own computed probability distribution $\{P_j\}_{j\in S}$.
% The topological path planning for computing a total-cost lowering \emph{path combination} that are to be used as reference paths in different homology classes by the robot is executed \changedC{only once at the beginning, and then every robot sticks to the assigned class.}
Once a robot chooses its own \emph{reference path}, it commits to that path, since without inter-robot coordination and without live global traffic updates, there is not new information to warrant a full-blown replanning of reference path.
Each robot then starts executing its reference path while performing a \emph{fast replanning} (described in Section~\ref{sec:cost-heuristic}) at regular intervals of time with an appropriately chosen heuristic function 
% \changedC{every second after that}, in which they use the appropriate heuristic function 
in order
to follow the reference path while avoiding collision with pedestrians and other robots.
The fast replanning is performed using A* search in the $H_2$-augmented graph, $\mathcal{G}_{H_2}$, and does not compute the path choice probabilities, but simply avoids high pedestrian/robot density regions as estimated in the immediate future in the spatio-temporal domain using sensing of the immediate vicinity. Feedback linearization and a potential-based approach allows control of the robot while avoiding collision.
The following subsections give more details \changedF{on prediction of agent density in the immediate spatiotemporal neighborhood (Section~\ref{sec:probability-prediction}), the fast replanning algorithm (Section~\ref{sec:cost-heuristic}), and a potential-based local collision avoidance for the non-holonomic robot model (Section~\ref{sec:collision-avoidance})}.

% \changedC{Once a robot chooses its own \emph{reference path}, it commits to that path (without inter-robot coordination and without live global traffic updates, there is not new information to warrant a full-blown replanning of reference path).
% % an appropriate path combination is chosen, each robot uses its respective path %(a probability-weight travel time minimizing path in respective homology class) 
% % as a \emph{reference path} to follow. 
% However, the exact path that a robot follows needs 
% % in experiments and simulations 
% to be re-planned in 
% \changedSB{order to avoid collision with pedestrians and other robots.}
% % presence of unpredictable agents such as pedestrians as well as to avoid collision with other robots, during which the robots compute their path in a certain order so that the latter robots are always aware of former ones' near-future locations (for collision avoidance). 
% We use A* 
% % search algorithm implemented in DOSL library~\citep{dosl} in C++ programming language for searching 
% in the $H_2$-augmented graph, $\mathcal{G}_{H_2}$, for fast replanning. The following subsections give more details.}

\subsection{Spatio-temporal Representation of Other Agents' Near-future Occupancy Probability Distribution}
\label{sec:probability-prediction}

\begin{figure}
	\centering
	\includegraphics[width=0.95\columnwidth, trim=0 0 0 0, clip=true]{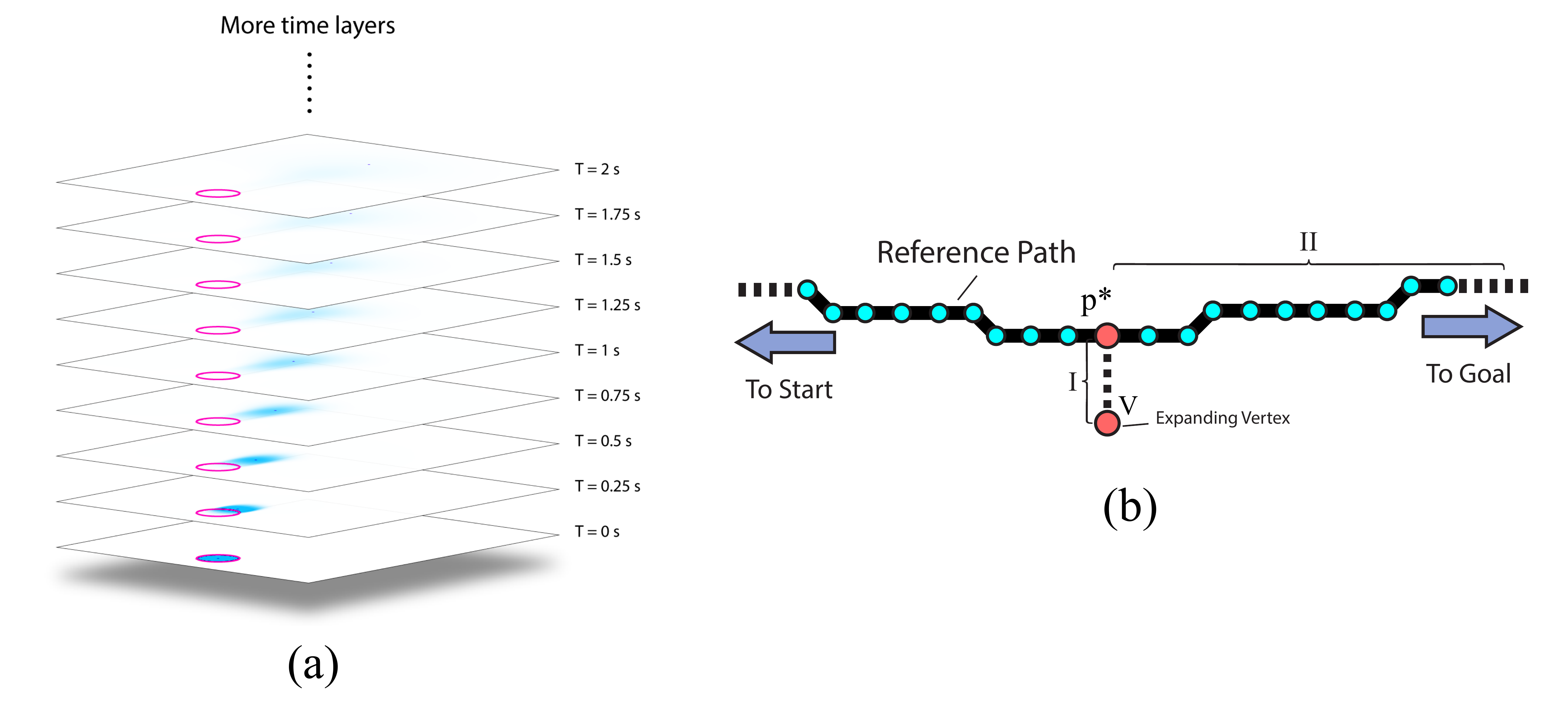}
	\caption{Fast replanning:
		\textbf{(a)} Illustration of \changedF{computation of a nearby agent's} probability in X-Y-T configuration space with $\delta_t=0.25 s$. The lighter cyan, the less likely it is occupied by \changedF{the agent}. The red ring shows the region (including safety radius) occupied by the \changedF{agent} at the time of prediction.
		\textbf{(b)} Heuristic function for fast re-planning returns the estimated time of travel for the parts I and II.}
\subfiglabel{pic_timelayers}{(a)}
\subfiglabel{pic_heuristic_re-planning}{(b)}
\end{figure}

\changedSB{For each agent in its immediate vicinity, a robot performs a short-term prediction of the agent's occupancy probability distribution as a density function in its spatio-temporal configuration space. 
Given the instantaneous position and velocity of a nearby agent (estimated using sensors onboard the robot),
the robot employs a simple prediction-only Markov localization approach~\citep{FOX1998195,zhan2021frontiers} (a discrete analog of Kalman filter) that uses a motion model to predict the probability of occupancy distribution 
(in a uniform discrete representation of the spatio-temporal domain) 
of the agent for the next $m_{\text{max}}$ timesteps
(with each times-step of length $\delta t$, and a spatial discretization of $\delta r'$ -- the same discretization used in construction of $\mathcal{G}$ --
% using the same discrete representation of the spatio-temporal domain as was used to construct $\mathcal{G}$
Section~\ref{sec:discretization}).}
% 
% The 3-D spatio-temporal configuration space \changed{is composed of} of multiple discrete time layers (shown in Figure \ref{pic_timelayers}) that are uniformly separated from each other by $\delta t$. 
%Knowing the current positions of a pedestrian (hence the obstacle map in the $S-K$ domain) and estimates of its speed and trajectory curvature ($\mu_V,\mu_K,\Sigma$),
%we predict its probability distributions, $\rho'_{\Delta t}$, at times $\Delta t = m~\delta t, ~m=1,2,\cdots,m_\text{max}$ in the future using the method described above (see Figure~\ref{pic_timelayers}).
\changedC{
% We predict its probability distributions, $\rho'_{\Delta t}$, at times $\Delta t = m~\delta t, ~m=1,2,\cdots,m_\text{max}$ in the future.
These occupancy probability values from the different nearby agents are aggregated (point-wise maximum) to construct the probability of occupancy map, $\mathcal{P}(x,y,t)$, that assigns a value to every discrete cell \changedF{(Figure \ref{pic_timelayers})}.
% in discrete spatio-temporal representation used to compute the search graph $\mathcal{G}$.
% in computing the cost function during the re-planning of A* searches. However, 
In practice, the probability computations are done on-the-fly during the graph search and only for $t$ between the current time and $m_{\text{max}}$ time-steps into the future.
% with the probability of occupany being computed 
% for fast re-planning, with $\rho'_{\Delta t}$ being computed only when a vertex in a time layer $\Delta t$ in the future is expanded.
% \todo{Xiaolong: What values of $\delta t, \delta r', and m_{max}$ did you use?}
% In spite of many methods for human motion prediction, to simplify and speed up the on-the-fly prediction process during the re-planning cycles, we use pedestrians' blurred ground truth in only near future. This ranges from 3 to 7 seconds, depending on how long a re-planning circle costs.
%For larger values of $m$ the probability density ``blurs'' into an almost uniform distribution, thus allowing us to skip these probability computations for time layer beyond $\Delta t = m_\text{max} \delta t$.
}

\subsection{Fast Replanning -- Heuristic Function and Cost Function}
\label{sec:cost-heuristic}

\changedC{In order to perform a fast re-planning to avoid collision with other \changedF{nearby} agents, while ensuring that a robot stays committed to its reference path, we design a \emph{reference-path-based heuristic function} for guiding an A* search on the $H_2$-augmented graph $\mathcal{G}_{H_2}$ for quickly computing a path in the same homology class as the reference path.}
%reference is set, while expanding vertices in the configuration space, 
The heuristic function \changedB{for} the fast re-planner \changedF{for a vertex $(\mathbf{v},t,h)\in\mathcal{V}_{H_2}$ is computed as follows (Figure~\ref{pic_heuristic_re-planning}):}
% (\changedB{illustrated} in Figure \ref{pic_heuristic_re-planning}), 
% computes the estimated time taken to first reach the closest point on the reference path, $\pi_{\text{ref}}$ (described as a sequence of points on the planar domain), and then to follow the path to goal.
% \changedB{when expanding a vertex $(\mathbf{v},t,h)\in\mathcal{V}_{H_2}$ in the $H_2$-augmented graph,} 
We compute the closest vertex $\mathbf{p}^* = \arg\!\min_{\mathbf{p}\in \pi_{\text{ref}}} \|\mathbf{x}-\mathbf{p}\|$ \changedB{on} the reference path, \changedF{$\pi_{\text{ref}}$ (described as a sequence of points on the planar domain)}, and return the \changedB{Euclidean distance between $\mathbf{v}$ and $\mathbf{p}^*$} (\changedB{referred to as \emph{part I} of the heuristic function}), \changedF{and add to it} the cost \changedC{(travel time)} from $\mathbf{p}^*$ to the goal \changedB{on the reference path} (referred to as \emph{part II} of the heuristic function) which was \changedC{pre-computed as part of the reference path search in $H_2$-augmented graph.}
\changedB{More formally, the heuristic function evaluated at $(\mathbf{x},t,h)\in \mathcal{V}_{H_2}$ is described as $h_{\pi_{\text{ref}}}(\mathbf{x},t,h) = \alpha \left(  \frac{ \|\mathbf{x}-\mathbf{p}^*\|}{V_\text{max}} + C_{\pi_{\text{ref}}}(\mathbf{p}^*) \right)$,
where 
% $D_\text{Eu}$ is the Euclidean distance, 
% $\mathbf{p}^* = \arg\!\min_{\mathbf{p}\in \pi_{\text{ref}}} \|\mathbf{x}-\mathbf{p}\|$ and $C_{\pi_{\text{ref}}}$ refers to the \changedC{cost (travel time)} from $\mathbf{p}^*$ along $\pi_{\text{ref}}$ which was \changedC{pre-computed as part of the reference path search in $H_2$-augmented graph.} %In the $H_2$-augmented graph, this heuristic function is used on the
$\alpha\leq 1$ is a constant to tune the inadmisibility of the heuristic function, with lower value of $\alpha$ %(taking the heuristic closer to zero) 
allowing greater deviation of the re-planned path from the reference path. % (and hence returning closer-to-optimal solution in that homology class since depending on the state of the pedestrians and the exact location of the robot the optimal path will not coincide exactly with the reference path, which was optimal only at the stage of initial planning of paths in the different homology classes). 
}

% \subsection{Cost Function in Presence of Uncertain Pedestrians} 

%\changedC{For the re-planning of the paths in the same homology classes with the reference path for a particular robot,}
%the cost function used is a combination of the travel time and the Euclidean length. 
The cost function for fast replanning
\changedSB{not only tries to minimize travel time, but also accounts for the computed nearby agent probability of occupancy, $\mathcal{P}$. In particular, the cost of an edge, $e\in \mathcal{E}$, connecting two points in the spatio-temporal domain is described by $C_{\mathcal{G}}(e) = %\sqrt{\epsilon V_\text{max}^2+1}
\int_{\overline{e}}  \frac{1+\iota}{1+\iota-\mathcal{P}(x,y,t)} \, dt$, where the integration is a line integration on the segment representing the edge and is performed numerically using linear interpolation of $\mathcal{P}$ along the uniformly discretized segment, and $\iota=0.001$ is a small positive constant used for numerical stability. Note that if the probability of occupancy is close to $1$ at some point, that point will have very high cost and will hence be avoided.}

% \import{version_3_section_3b_c}
\subsection{Non-holonomic Robot Control for Trajectory Tracking and Potential-based Collision Avoidance}
\label{sec:collision-avoidance}

\changedF{\emph{Trajectory tracking:} Each robot tracks its computed trajectory (the reference trajectory at the beginning, and the fast-replanned trajectory subsequently).}
In the experiments, each non-holonomic differential drive robot controls a \emph{lookahead point}~\citep{d1992dynamic} by computing the corresponding linear and angular velocities. % for the robot are calculated accordingly and converted into rotate speed for each wheel.
\changedF{At a given time step, a target point $\mathbf{p}_l=(x_l,y_l)$ is extracted from the time parametrized trajectory output by the re-planning algorithm. This set point is tracked using a feedback linearization controller designed as follows: Given the current robot position $\mathbf{p}_c=(x_c,y_c)$, orientation $\theta_c$ and a constant lookahead distance $d_f$, a Cartesian velocity $\mathbf{v}=(v_x,v_y)=c_s (\mathbf{p}_l - \mathbf{p}_d)$ is computed for a point, $\mathbf{p}_d=(x_d,y_d)=(x_c+d_f\cos\theta_c,y_c+d_f\sin\theta_c)$, which is a point at a distance $d_f$ ahead of the robot.
This enables the lookahead point track the target point $\mathbf{p}_l$. $c_s$ is a constant scalar gain. Corresponding linear and angular velocities for the robot is thus computed as $u=v_x\cos\theta_c + v_y\sin\theta_c$ and $\omega=(v_y\cos\theta_c - v_x\sin\theta_c)/d_f$. For the differential drive robot, linear and angular velocities can be mapped to left and right wheel efforts (or velocities) as $v_L=u-\omega l/2$ and $v_R=u+\omega l/2$ where $l$ is the separation between the two wheels of the robot.}

\emph{Local collision avoidance with other agents:}
Depending on the robot heading, a \emph{fan-shaped collision cone} is generated in front of the robot with radius $r_c$ and angle $\alpha_c$. For every other agent, $i$, detected inside the collision cone with position $\mathbf{p}_{\text{agnt},i}=(x_{\text{agnt},i},y_{\text{agnt},i})$, a repulsion velocity
to slow down the robot, that is inversely proportional to the distance between them,
is computed as $\mathbf{v}_{\text{rep},i} = -c_\text{a}\, (\mathbf{p}_{\text{agnt},i} - \mathbf{p}_c)/\|\mathbf{p}_{\text{agnt},i} - \mathbf{p}_c\|^2$, where $c_\text{a}$ is a positive constant.
The resulting Cartesian velocity of the robot is computed as $\mathbf{v}_{\text{avoid}} = \mathbf{v} + \sum_{i} \mathbf{v}_{\text{agnt},i}$.
\changedF{This emulates the behavior of a vehicle that tries to avoid other agents ahead of it, but not behind it.

\emph{Local collision avoidance with obstacles:}
For avoiding robot-environment collision (including environment boundaries and obstacles) we use a velocity cancellation policy as follows: The closest point on an obstacle or environment boundary to the robot is denoted as $\mathbf{p}_\text{env}=(x_\text{env},y_\text{env})$.
Then the vector from the robot to the environment is defined as $\mathbf{v}_\text{env}= \mathbf{p}_\text{env}-\mathbf{p}_c$.
% If the robot is closer to the environment than some threshold $d_e$, such that $||v_e|| < d_e$, and it has a velocity component towards the environment such that $v \cdot v_e > 0$, that component of the velocity is cancelled as $v' = v - (v \cdot v_e) v_e$.
An obstacle repulsion component of the velocity is activated only if the obstacle is sufficiently close, and the robot has a velocity component \emph{towards} the obstacle,
% (\emph{i.e.} the robot is potentially moving towards the obstacle)
and thus the final velocity of the robot is computed as follows:

\indent $\mathbf{v}_{\text{final}} = \left\{\begin{array}{ll}
\mathbf{v}_{\text{avoid}} - c_\text{e}\, (\mathbf{v}_{\text{avoid}} \cdot \mathbf{v}_\text{env}) \mathbf{v}_\text{env}/ \|\mathbf{v}_\text{env}\|^2,  & ~~\text{if } \|\mathbf{v}_\text{env}\| < d_e \text{ and } \mathbf{v} \cdot \mathbf{v}_\text{env} > 0, \\
\mathbf{v}_{\text{avoid}}, & ~~\text{otherwise.}
\end{array}\right.$

\noindent
where, $c_\text{e}$ and $d_e$ are positive constants.}

\section{Results \& Discussions}
\label{sec:results}
%\subsection{Input}
\changed{
%\changedC{Two probability distribution models, the 2-Robot model (Total) with the \emph{qpOASES} \citep{Ferreau2014} for quadratic programming and the Ensemble model (Max) with \emph{NLopt} \citep{NLopt} for non-linear programming}, are applied to the same setups to show the performance difference.
\changedC{We run the simulations on several maps: ``\emph{cage\_1}'' (Figure \ref{pic_classes}), ``\emph{cage\_2}'' (Figure \ref{pic_density_map_1}), ``\emph{lehigh}'' (Figure \ref{pic_density_map_2}), \changedXW{``\emph{o2}'' (Figure \ref{pic_o2}), and ``\emph{group}'' (Figure \ref{pic_group})}.
%Map cage\_1 and cage\_2 are in lab size; Map lehigh is in real-world size. 
% \changedSB{All distant agents are modeled as pedestrians and their trajectories are randomly generated.} 
% 
\changedXW{Three types of comparisons are made:}
\begin{enumerate}
% In the `\emph{lehigh}'' map pedestrians are allowed to move through the obstacles.}
\item[i.] \changed{We \changedXW{first} compare our proposed \emph{topological planning} algorithm (a robot stochastically choosing paths from available topological classes) with a \emph{shortest-path} algorithm (each robot, without any inter-robot coordination, chooses the shortest path to goal).
% for \changedXW{the first three} setups.
In this comparison all distant agents are modeled as pedestrians and their trajectories are randomly generated.
% 
%; in the latter, during the initial planning, each robot simply chooses the shortest collision-free path to the goal instead of reasoning about topological classes).
}

\item[ii.]
\changedXW{Then, 
\changedF{we demonstrate the effectiveness of \emph{a.} the traffic density map, and, \emph{b.} the proposed computation of path choice probabilities by comparing the performance of our algorithm with versions that either does use an uniform traffic density map or uniformly path choice probabilities over the topological classes.}
% we compare our proposed topological method with a variant that either uses uniform traffic density map or uniformly distribute probabilities over the classes.
\item[iii.]
Finally, we apply our proposed method to a setup with multiple groups of robots that start from different locations and have different goal locations, and compare the performance of our method with the performance of the shortest-path algorithm.}
\end{enumerate}

\noindent
\revisedXW{It is worth noting that the fundamental premise of lack of inter-agent coordination or communication makes our work extremely unique. We assume that there exists no inter-agent communication or coordination, and robots do not share their plan or intent with other agents or with any central server. No other prior work, to our knowledge, assumes complete lack of communication or coordination (for example, in~\cite{street2021congestion} there exists communication and coordination between the robots in construction of a shared PRT). Hence a fair comparison of our method with such alternatives in literature is not possible.}

\vspace{0.5em}\noindent
\textbf{i.} 
\changedXW{In the topological-versus-shortest-path comparison, the robots are allowed to choose one out of up to $m=6$ classes.}
\changedC{
\changedSB{In each environment we vary the number of robots, $n$, and the number of pedestrians, $Q$, and note the average travel time of the robots and the maximum travel time (the time taken by the last robot to reach its goal). We also measure the average time spent on collision avoidance per robot. A simulation with the same initial conditions is performed using each of the proposed topological algorithm and the shortest path algorithm. %The ratio of each of these
}}
\changedC{Tables~\ref{tab_performance_comparison} and \ref{tab_performance_comparison_2} shows a performance comparison. Each of the percentage numbers is the ratio of travel time (average or maximum) between the simulations using the topological algorithm and that using the shortest path algorithm. For collisions we show the difference between the time spent avoiding collisions using the topological algorithm and that using the shortest path algorithm\footnote{in order to avoid divisions by zero, we choose not to compute percentage values.}
In computing the path assignment probability values for the topological algorithm we can use either the maximum travel time cost, $C_{\text{max}}$, or the average travel time cost, $C_{\text{avg}}$, and choose one out of the two simplified formulations -- \emph{2-robot model} (solved using QP library `qpOASES') or the \emph{ensemble model} (solved using NLP library `NLopt'). This is indicated in the first column of the tables.
% This is indicated in the 

\highlight{As evident from the results, as the number of robots and pedestrians increase (\emph{i.e.}, the potential of congestion increases), our proposed topological algorithm significantly outperforms the shortest path algorithm in all aspects. 
It is also worth noting that in the larger \emph{lehigh} map, using the travel time costs $C_{\text{max}}$, the advantage is higher in the maximum travel time than the average travel time.}

%\vspace{-0.05in}
\begin{table}[ht] %\vspace{-0.5em}
\centering
\caption{Performance of 
\changedSB{our proposed topological algorithm as compared to a shortest path algorithm}
in simulations for map ``cage\_1'', using two time-cost \& assignment probability computation models. 
Darker cyan indicates bigger advantage of the topological algorithm over the shortest-path one, while darker red indicates cases where it underperformed. \changedSB{Each cell shows the the average over 10 simulation runs with different initial conditions. See multimedia attachment for sample simulation runs.}} \label{tab_performance_comparison}
%\resizebox{\columnwidth}{!}{
\begin{threeparttable}%\vspace{0.1in}   	
\centering
\begin{tabular}{c|c|r|rrrr} %\vspace{-0.5em}
\bottomrule
\multirow{2}{*}{Method} & \multirow{2}{*}{Feature} & \multirow{2}{*}{\begin{tabular}[c]{@{}c@{}}Ped.\\ \#\end{tabular}}& \multicolumn{4}{c}{Robot \#}\\
\cline{4-7}
& & & 5 & 10 & 15 & 20\\ 
\toprule
\multirow{21}{*}{\rotatebox[origin=c]{90}{\small \textbf{2-robot Model}, minimizing \textbf{Avg. Travel Time Cost}}} & \multirow{7}{*}{\rotatebox[origin=c]{90}{Avg. travel time}} & 0 & \cellcolor[HTML]{E4F7FB}95.27\% & \cellcolor[HTML]{BDF0F6}87.67\%  & \cellcolor[HTML]{99E9F1}80.52\% & \cellcolor[HTML]{80E4EE}75.44\% \\
&  & 5 & \cellcolor[HTML]{DCF6FA}93.78\% & \cellcolor[HTML]{B8EFF5}86.59\%  & \cellcolor[HTML]{A1EAF2}82.05\% & \cellcolor[HTML]{70E1EC}72.41\% \\
&  & 10 & \cellcolor[HTML]{E0F6FB}94.60\% & \cellcolor[HTML]{97E8F1}80.00\%  & \cellcolor[HTML]{91E7F0}78.87\% & \cellcolor[HTML]{6AE0EB}71.12\% \\
&  & 15 & \cellcolor[HTML]{CDF3F8}90.68\% & \cellcolor[HTML]{87E5EF}76.95\%  & \cellcolor[HTML]{54DCE8}66.80\% & \cellcolor[HTML]{35D6E4}60.63\% \\
&  & 20 & \cellcolor[HTML]{8FE7F0}78.51\% & \cellcolor[HTML]{FCE6EE}104.11\% & \cellcolor[HTML]{7AE3ED}74.36\% & \cellcolor[HTML]{71E1EC}72.50\% \\
&  & 25 & \cellcolor[HTML]{DEF6FA}94.08\% & \cellcolor[HTML]{FEA7BE}115.45\% & \cellcolor[HTML]{D9F5FA}93.12\% & \cellcolor[HTML]{70E1EC}72.28\% \\
&  & 30 & \cellcolor[HTML]{98E9F1}80.32\% & \cellcolor[HTML]{50DBE7}65.95\%  & \cellcolor[HTML]{1FD1E1}56.22\% & \cellcolor[HTML]{16D0E0}54.44\% \\
\hhline{~------}
& \multirow{7}{*}{\rotatebox[origin=c]{90}{Max. travel time}} & 0 & \cellcolor[HTML]{CDF3F8}90.72\% & \cellcolor[HTML]{9BE9F1}80.78\%  & \cellcolor[HTML]{76E2EC}73.41\% & \cellcolor[HTML]{57DCE8}67.27\% \\
&  & 5 & \cellcolor[HTML]{CCF2F8}90.60\% & \cellcolor[HTML]{A2EAF2}82.18\%  & \cellcolor[HTML]{7FE4EE}75.30\% & \cellcolor[HTML]{60DEEA}69.17\% \\
&  & 10 & \cellcolor[HTML]{CAF2F8}90.19\% & \cellcolor[HTML]{67DFEA}70.53\%  & \cellcolor[HTML]{7DE3ED}74.91\% & \cellcolor[HTML]{65DFEA}70.06\% \\
&  & 15 & \cellcolor[HTML]{AAECF4}83.84\% & \cellcolor[HTML]{78E2ED}73.83\%  & \cellcolor[HTML]{3BD7E5}61.88\% & \cellcolor[HTML]{2AD4E2}58.50\% \\
&  & 20 & \cellcolor[HTML]{88E5EF}77.07\% & \cellcolor[HTML]{FDBED0}111.23\% & \cellcolor[HTML]{85E5EF}76.51\% & \cellcolor[HTML]{6EE1EB}71.93\% \\
&  & 25 & \cellcolor[HTML]{F6FAFE}98.84\% & \cellcolor[HTML]{FF8DAA}128.70\% & \cellcolor[HTML]{EBF8FC}96.81\% & \cellcolor[HTML]{69E0EB}70.97\% \\
&  & 30 & \cellcolor[HTML]{A4EBF3}82.64\% & \cellcolor[HTML]{46D9E6}63.98\%  & \cellcolor[HTML]{00CCDD}48.89\% & \cellcolor[HTML]{10CFDF}53.23\% \\
\hhline{~------}
& \multirow{7}{*}{\rotatebox[origin=c]{90}{Collision (s)}} & 0  & \cellcolor[HTML]{FBFBFE}-0.01 & \cellcolor[HTML]{FAFBFE}-0.05 & \cellcolor[HTML]{F3FAFD}-0.35 & \cellcolor[HTML]{EFF9FD}-0.49 \\
&  & 5  & \cellcolor[HTML]{FBFBFE}-0.03 & \cellcolor[HTML]{F8FBFE}-0.13 & \cellcolor[HTML]{E5F7FB}-0.88 & \cellcolor[HTML]{D2F4F9}-1.63 \\
&  & 10 & \cellcolor[HTML]{FAFBFE}-0.05 & \cellcolor[HTML]{F9FBFE}-0.10 & \cellcolor[HTML]{DDF6FA}-1.22 & \cellcolor[HTML]{D5F4F9}-1.54 \\
&  & 15 & \cellcolor[HTML]{FAFBFE}-0.08 & \cellcolor[HTML]{F5FAFE}-0.27 & \cellcolor[HTML]{D4F4F9}-1.57 & \cellcolor[HTML]{BBEFF6}-2.55 \\
&  & 20 & \cellcolor[HTML]{F6FBFE}-0.21 & \cellcolor[HTML]{99E9F1}-3.92 & \cellcolor[HTML]{A6EBF3}-3.41 & \cellcolor[HTML]{79E3ED}-5.17 \\
&  & 25 & \cellcolor[HTML]{F7FBFE}-0.20 & \cellcolor[HTML]{81E4EE}-4.86 & \cellcolor[HTML]{7FE4EE}-4.95 & \cellcolor[HTML]{73E2EC}-5.41 \\
&  & 30 & \cellcolor[HTML]{F3FAFD}-0.34 & \cellcolor[HTML]{F7FBFE}-0.17 & \cellcolor[HTML]{DCF6FA}-1.24 & \cellcolor[HTML]{B4EEF5}-2.84 \\
\hline
\multirow{21}{*}{\rotatebox[origin=c]{90}{\small \textbf{Ensemble Model}, minimizing \textbf{Max. Travel Time Cost}}} & \multirow{7}{*}{\rotatebox[origin=c]{90}{Avg. travel time}} & 0  & \cellcolor[HTML]{FDBCCE}111.56\% & \cellcolor[HTML]{CAF2F8}90.25\%  & \cellcolor[HTML]{A1EAF2}82.04\% & \cellcolor[HTML]{80E4EE}75.55\% \\
&  & 5  & \cellcolor[HTML]{FEA8BF}115.29\% & \cellcolor[HTML]{CBF2F8}90.41\%  & \cellcolor[HTML]{97E8F1}79.99\% & \cellcolor[HTML]{6EE1EB}71.99\% \\
&  & 10 & \cellcolor[HTML]{FDCFDD}108.19\% & \cellcolor[HTML]{A8ECF3}83.36\%  & \cellcolor[HTML]{A6EBF3}82.94\% & \cellcolor[HTML]{63DFEA}69.79\% \\
&  & 15 & \cellcolor[HTML]{FBFBFE}99.85\%  & \cellcolor[HTML]{78E3ED}74.00\%  & \cellcolor[HTML]{65DFEA}70.07\% & \cellcolor[HTML]{3ED7E5}62.32\% \\
&  & 20 & \cellcolor[HTML]{C9F2F8}89.91\%  & \cellcolor[HTML]{F5FAFE}98.68\%  & \cellcolor[HTML]{72E1EC}72.78\% & \cellcolor[HTML]{61DEEA}69.40\% \\
&  & 25 & \cellcolor[HTML]{E2F7FB}94.89\%  & \cellcolor[HTML]{FDC7D6}109.67\% & \cellcolor[HTML]{BEF0F6}87.84\% & \cellcolor[HTML]{7BE3ED}74.45\% \\
&  & 30 & \cellcolor[HTML]{A1EAF2}82.13\%  & \cellcolor[HTML]{52DBE8}66.33\%  & \cellcolor[HTML]{30D5E3}59.58\% & \cellcolor[HTML]{19D0E0}55.01\% \\
\hhline{~------}
& \multirow{7}{*}{\rotatebox[origin=c]{90}{Max. travel time}} & 0  & \cellcolor[HTML]{FDC8D7}109.48\% & \cellcolor[HTML]{A8ECF3}83.52\%  & \cellcolor[HTML]{78E2ED}73.95\% & \cellcolor[HTML]{54DCE8}66.68\% \\
&  & 5  & \cellcolor[HTML]{FDB9CC}112.16\% & \cellcolor[HTML]{B1EDF4}85.17\%  & \cellcolor[HTML]{7AE3ED}74.40\% & \cellcolor[HTML]{47D9E6}64.25\% \\
&  & 10 & \cellcolor[HTML]{FCF9FD}100.58\% & \cellcolor[HTML]{81E4EE}75.71\%  & \cellcolor[HTML]{B0EDF4}85.01\% & \cellcolor[HTML]{4DDAE7}65.39\% \\
&  & 15 & \cellcolor[HTML]{C9F2F8}89.90\%  & \cellcolor[HTML]{62DEEA}69.52\%  & \cellcolor[HTML]{4ADAE7}64.88\% & \cellcolor[HTML]{3CD7E5}61.92\% \\
&  & 20 & \cellcolor[HTML]{EBF8FC}96.68\%  & \cellcolor[HTML]{FCE1EA}105.01\% & \cellcolor[HTML]{7EE4EE}75.10\% & \cellcolor[HTML]{5ADDE9}68.03\% \\
&  & 25 & \cellcolor[HTML]{F3FAFD}98.23\%  & \cellcolor[HTML]{FF8DAA}124.14\% & \cellcolor[HTML]{BBEFF6}87.29\% & \cellcolor[HTML]{72E1EC}72.70\% \\
&  & 30 & \cellcolor[HTML]{A6EBF3}82.95\%  & \cellcolor[HTML]{38D6E4}61.12\%  & \cellcolor[HTML]{0ECEDE}52.84\% & \cellcolor[HTML]{0BCEDE}52.27\% \\
\hhline{~------}
& \multirow{7}{*}{\rotatebox[origin=c]{90}{Collision (s)}} & 0  & \cellcolor[HTML]{FBFBFE}-0.01 & \cellcolor[HTML]{FAFBFE}-0.05 & \cellcolor[HTML]{F4FAFD}-0.31 & \cellcolor[HTML]{F3FAFD}-0.34 \\
&  & 5  & \cellcolor[HTML]{FAFBFE}-0.04 & \cellcolor[HTML]{F9FBFE}-0.08 & \cellcolor[HTML]{E2F7FB}-1.01 & \cellcolor[HTML]{D6F4F9}-1.50 \\
&  & 10 & \cellcolor[HTML]{FAFBFE}-0.05 & \cellcolor[HTML]{F8FBFE}-0.15 & \cellcolor[HTML]{DFF6FB}-1.13 & \cellcolor[HTML]{D8F5FA}-1.42 \\
&  & 15 & \cellcolor[HTML]{F9FBFE}-0.10 & \cellcolor[HTML]{F3FAFD}-0.33 & \cellcolor[HTML]{D7F4FA}-1.46 & \cellcolor[HTML]{C9F2F8}-2.02 \\
&  & 20 & \cellcolor[HTML]{F7FBFE}-0.19 & \cellcolor[HTML]{9AE9F1}-3.89 & \cellcolor[HTML]{A8ECF3}-3.30 & \cellcolor[HTML]{87E5EF}-4.62 \\
&  & 25 & \cellcolor[HTML]{F6FAFE}-0.21 & \cellcolor[HTML]{80E4EE}-4.91 & \cellcolor[HTML]{7BE3ED}-5.11 & \cellcolor[HTML]{6DE0EB}-5.67 \\
&  & 30 & \cellcolor[HTML]{F4FAFD}-0.30 & \cellcolor[HTML]{F0F9FD}-0.45 & \cellcolor[HTML]{E2F7FB}-1.00 & \cellcolor[HTML]{A1EAF2}-3.60 \\
\toprule
\end{tabular}
\begin{tablenotes}
\item[*] All methods use $a=0.1625$, $b=0.04548$.
\item \% values: $\frac{\text{time taken in topological algorithm}}{\text{time taken in shortest path algorithm}}\times 100\%$
\item Collision numbers: {(colliding duration per robot in topological algorithm) - (colliding duration per robot in shortest-path algorithm)}
\end{tablenotes} %\vspace{-1em}
\end{threeparttable}
%}
\end{table}

\begin{table}
\caption{Performance comparisons in the other two maps.
% in simulations.
}
\label{tab_performance_comparison_2}
\centering
%\resizebox{\columnwidth}{!}{
\begin{tabular}{c|c|r|rrrr}
\bottomrule
\multirow{2}{*}{Method} & \multirow{2}{*}{Feature} & \multirow{2}{*}{Ped. \# %\begin{tabular}[c]{@{}c@{}}Ped.\\ \#\end{tabular}
}& \multicolumn{4}{c}{Robot \#}\\
\cline{4-7}
& & & 5 & 10 & 15 & 20\\ 
\toprule
\multicolumn{7}{c}{Map: cage\_2} \\
\hline
\multirow{15}{*}{\rotatebox[origin=c]{90}{\textbf{$2$-robot Model}, \textbf{Avg. Time Cost}}} & \multirow{5}{*}{\rotatebox[origin=c]{90}{Avg (\%)}} & 0 & \cellcolor[HTML]{E7F8FC}95.85\% & \cellcolor[HTML]{CBF2F8}90.29\% & \cellcolor[HTML]{AFEDF4}84.85\% & \cellcolor[HTML]{9BE9F1}80.81\% \\
&  & 5 & \cellcolor[HTML]{B8EFF5}86.65\% & \cellcolor[HTML]{94E8F0}79.37\% & \cellcolor[HTML]{83E4EE}76.02\% & \cellcolor[HTML]{5CDDE9}68.26\% \\
&  & 10 & \cellcolor[HTML]{B5EEF5}85.97\% & \cellcolor[HTML]{62DEEA}69.49\% & \cellcolor[HTML]{65DFEA}70.12\% & \cellcolor[HTML]{39D6E4}61.40\% \\
&  & 15 & \cellcolor[HTML]{82E4EE}75.84\% & \cellcolor[HTML]{78E2ED}73.91\% & \cellcolor[HTML]{2AD4E2}58.37\% & \cellcolor[HTML]{24D2E1}57.25\% \\
&  & 20 & \cellcolor[HTML]{87E5EF}76.86\% & \cellcolor[HTML]{3BD7E5}61.80\% & \cellcolor[HTML]{08CDDE}51.63\% & \cellcolor[HTML]{08CDDE}51.70\% \\
\hhline{~------}
& \multirow{5}{*}{\rotatebox[origin=c]{90}{Max (\%)}} & 0 & \cellcolor[HTML]{D9F5FA}93.18\% & \cellcolor[HTML]{B3EEF5}85.70\% & \cellcolor[HTML]{8FE7F0}78.55\% & \cellcolor[HTML]{7CE3ED}74.66\% \\
&  & 5 & \cellcolor[HTML]{B5EEF5}86.11\% & \cellcolor[HTML]{95E8F1}79.69\% & \cellcolor[HTML]{64DFEA}69.92\% & \cellcolor[HTML]{47D9E6}64.10\% \\
&  & 10 & \cellcolor[HTML]{ADEDF4}84.49\% & \cellcolor[HTML]{51DBE7}66.14\% & \cellcolor[HTML]{62DEEA}69.57\% & \cellcolor[HTML]{3ED7E5}62.41\% \\
&  & 15 & \cellcolor[HTML]{99E9F1}80.55\% & \cellcolor[HTML]{84E5EE}76.35\% & \cellcolor[HTML]{27D3E2}57.79\% & \cellcolor[HTML]{04CCDD}50.84\% \\
&  & 20 & \cellcolor[HTML]{86E5EF}76.73\% & \cellcolor[HTML]{2AD4E2}58.49\% & \cellcolor[HTML]{0FCFDF}53.16\% & \cellcolor[HTML]{03CCDD}50.63\% \\
\hhline{~------}
& \multirow{5}{*}{\rotatebox[origin=c]{90}{Collision (s)}} & 0  & \cellcolor[HTML]{FAFBFE}-0.04 & \cellcolor[HTML]{FAFBFE}-0.05 & \cellcolor[HTML]{FCE3EC}0.23  & \cellcolor[HTML]{FDCFDD}0.41  \\
&  & 5  & \cellcolor[HTML]{FBFBFE}-0.02 & \cellcolor[HTML]{F8FBFE}-0.13 & \cellcolor[HTML]{FAFBFE}-0.05 & \cellcolor[HTML]{FDC7D6}0.48  \\
&  & 10 & \cellcolor[HTML]{FBFBFE}-0.01 & \cellcolor[HTML]{F8FBFE}-0.13 & \cellcolor[HTML]{FCDEE8}0.27  & \cellcolor[HTML]{F4FAFE}-0.28 \\
&  & 15 & \cellcolor[HTML]{F7FBFE}-0.18 & \cellcolor[HTML]{F6FAFE}-0.23 & \cellcolor[HTML]{FBFBFE}-0.02 & \cellcolor[HTML]{88E6EF}-4.58 \\
&  & 20 & \cellcolor[HTML]{F7FBFE}-0.18 & \cellcolor[HTML]{FCE6EE}0.21  & \cellcolor[HTML]{C4F1F7}-2.20 & \cellcolor[HTML]{4CDAE7}-6.96 \\
\toprule
\end{tabular}
%}
%\resizebox{\columnwidth}{!}{
\begin{tabular}{c|c|r|rrrrr}
% \bottomrule
% \multirow{2}{*}{\rotatebox[origin=c]{90}{Md.}} & \multirow{2}{*}{\rotatebox[origin=c]{90}{Ft.}} & \multirow{2}{*}{\begin{tabular}[c]{@{}c@{}}Ped.\\ \#\end{tabular}}& \multicolumn{6}{c}{Robot \#}\\
% \toprule
\multicolumn{8}{c}{Map: lehigh} \\
\hline
Method & Feature & Ped. \# & 10 & 20 & 30 & 40 & 50\\ 
\hline
\multirow{18}{*}{\rotatebox[origin=c]{90}{\textbf{Ensemble Model}, \textbf{Max. Time Cost}}} & \multirow{6}{*}{\rotatebox[origin=c]{90}{Avg (\%)}} & 0 & \cellcolor[HTML]{FCE5EE}104.20\% & \cellcolor[HTML]{F1FAFD}97.99\% & \cellcolor[HTML]{CFF3F8}91.13\% & \cellcolor[HTML]{CFF3F9}91.27\% & \cellcolor[HTML]{D2F4F9}91.77\% \\
&  & 20  & \cellcolor[HTML]{FCF6FA}101.19\% & \cellcolor[HTML]{DCF5FA}93.67\% & \cellcolor[HTML]{D1F3F9}91.53\% & \cellcolor[HTML]{A8ECF3}83.36\% & \cellcolor[HTML]{A0EAF2}81.93\% \\
&  & 40  & \cellcolor[HTML]{F2FAFD}98.09\%  & \cellcolor[HTML]{F0F9FD}97.74\% & \cellcolor[HTML]{C9F2F8}89.98\% & \cellcolor[HTML]{B0EDF4}84.98\% & \cellcolor[HTML]{97E8F1}80.06\% \\
&  & 60  & \cellcolor[HTML]{E5F7FC}95.61\%  & \cellcolor[HTML]{CFF3F9}91.20\% & \cellcolor[HTML]{A2EAF2}82.21\% & \cellcolor[HTML]{88E5EF}77.04\% & \cellcolor[HTML]{71E1EC}72.55\% \\
&  & 80  & \cellcolor[HTML]{FE8EAB}119.92\% & \cellcolor[HTML]{D5F4F9}92.39\% & \cellcolor[HTML]{9AE9F1}80.57\% & \cellcolor[HTML]{77E2ED}73.80\% & \cellcolor[HTML]{65DFEA}70.18\% \\
&  & 100 & \cellcolor[HTML]{F8FBFE}99.30\%  & \cellcolor[HTML]{B0EDF4}85.08\% & \cellcolor[HTML]{81E4EE}75.62\% & \cellcolor[HTML]{73E2EC}72.94\% & \cellcolor[HTML]{6BE0EB}71.27\% \\
\hhline{~-------}
& \multirow{6}{*}{\rotatebox[origin=c]{90}{Max (\%)}} & 0 & \cellcolor[HTML]{FDC3D4}110.32\% & \cellcolor[HTML]{E5F7FC}95.62\%  & \cellcolor[HTML]{B5EEF5}85.99\% & \cellcolor[HTML]{B4EEF5}85.78\% & \cellcolor[HTML]{C2F0F7}88.53\% \\
&  & 20  & \cellcolor[HTML]{FDD6E2}106.89\% & \cellcolor[HTML]{E5F7FB}95.49\% & \cellcolor[HTML]{C4F1F7}89.06\% & \cellcolor[HTML]{99E9F1}80.53\% & \cellcolor[HTML]{97E8F1}80.11\% \\
&  & 40  & \cellcolor[HTML]{FCDDE8}105.61\% & \cellcolor[HTML]{FCDFE9}105.41\% & \cellcolor[HTML]{C9F2F8}89.96\% & \cellcolor[HTML]{A0EAF2}81.80\% & \cellcolor[HTML]{87E5EF}76.91\% \\
&  & 60  & \cellcolor[HTML]{F5FAFE}98.76\%  & \cellcolor[HTML]{C8F2F8}89.79\%  & \cellcolor[HTML]{87E5EF}76.87\% & \cellcolor[HTML]{62DEEA}69.52\% & \cellcolor[HTML]{55DCE8}66.96\% \\
&  & 80  & \cellcolor[HTML]{FF8DAA}132.74\% & \cellcolor[HTML]{C9F2F8}89.90\%  & \cellcolor[HTML]{6FE1EC}72.13\% & \cellcolor[HTML]{52DBE8}66.38\% & \cellcolor[HTML]{35D6E4}60.62\% \\
&  & 100 & \cellcolor[HTML]{F7FBFE}99.15\%  & \cellcolor[HTML]{AFEDF4}84.85\% & \cellcolor[HTML]{5DDDE9}68.58\% & \cellcolor[HTML]{56DCE8}67.18\% & \cellcolor[HTML]{37D6E4}60.94\% \\
\hhline{~-------}
& \multirow{6}{*}{\rotatebox[origin=c]{90}{Collision (s)}} & 0 & \cellcolor[HTML]{FCFCFF}0.00  & \cellcolor[HTML]{FCF3F8}0.02  & \cellcolor[HTML]{FCFAFE}0.00  & \cellcolor[HTML]{FBFBFE}-0.01 & \cellcolor[HTML]{00CCDD}-3.04 \\
&  & 20  & \cellcolor[HTML]{FBFBFE}0.00 & \cellcolor[HTML]{FCFCFF}0.00  & \cellcolor[HTML]{CEF3F8}-0.55 & \cellcolor[HTML]{FAFBFE}-0.02 & \cellcolor[HTML]{FAFBFE}-0.02 \\
&  & 40  & \cellcolor[HTML]{FBFBFE}0.00  & \cellcolor[HTML]{E2F7FB}-0.30 & \cellcolor[HTML]{FF8DAA}0.25  & \cellcolor[HTML]{F7FBFE}-0.05 & \cellcolor[HTML]{F8FBFE}-0.05 \\
&  & 60  & \cellcolor[HTML]{FBFBFE}0.00  & \cellcolor[HTML]{FAFBFE}-0.02 & \cellcolor[HTML]{FBFBFE}-0.01 & \cellcolor[HTML]{F4FAFD}-0.09 & \cellcolor[HTML]{F4FAFE}-0.09 \\
&  & 80  & \cellcolor[HTML]{8AE6EF}-1.38 & \cellcolor[HTML]{FCF7FB}0.01  & \cellcolor[HTML]{F4FAFE}-0.09 & \cellcolor[HTML]{FCE2EB}0.06  & \cellcolor[HTML]{B2EDF5}-0.89 \\
&  & 100 & \cellcolor[HTML]{FBFBFE}-0.01 & \cellcolor[HTML]{FCEAF2}0.04  & \cellcolor[HTML]{F0F9FD}-0.14 & \cellcolor[HTML]{D6F4F9}-0.45 & \cellcolor[HTML]{6DE0EB}-1.72 \\
\toprule %\vspace{-2em}
\end{tabular} %\vspace{-3em}
%}
\end{table}

\clearpage

\vspace{0.5em}\noindent
\textbf{ii.} \emph{a.} 
\changedXW{To verify that the traffic density map (as described in Section~\ref{sec:traffic-density}) in our proposed algorithm makes a difference in performance, we used a uniform traffic density map to run 10 simulations for comparison (while keeping the rest of the algorithm the same). \highlight{The results in Table \ref{table:density_map} suggests that the algorithm without an appropriately computed traffic density map underperforms. The role of the traffic density map in predicting the traffic for computation of the path assignment probabilities is statistically significant.}

\vspace{0.25em}\noindent
\textbf{ii.} \emph{b.} 
To demonstrate the effectiveness of our proposed path assignment probability computation, we compare the performance of our \textbf{ensemble model} for probability computation with uniform path assignment probabilities.
% Furthermore, although assigning robots to more classes would help reduce congestion, it does require a reasonable probability distribution. In the competing algorithm, the probabilities of classes are even.
In particular, in map \emph{o2} (Figure~\ref{pic_o2}), the ensemble model minimizing the max. travel time cost gives path assignment probabilities of $0.34$, $0.25$, $0.26$ and $0.14$ for the $4$ topological classes in the environment, which is sufficiently different from uniform probability of $0.25$ for all the classes.
This makes this environment a prime candidate for comparison of our method with the uniform path assignment probability method.
% we compare the performene with an 
% there are 4 classes to choose, each one with a probability of 0.25. 
\highlight{Table \ref{table:uniform} shows that the path assignment probabilities computed using the proposed algorithm leads to an improved performance with a variety of robot and pedestrian setups, when compared against the uniform path assignment probability method.}

\begin{table}
\caption{Performance comparisons between simulations using traffic density map, $\rho$, and not using it (\changedF{\emph{i.e.}, effectively uniform traffic density} by setting traffic density term 0) on Map ``\emph{cage\_1}'', both with 6 classes, 10 robots, 10 pedestrians, and using the \textbf{2-robot model} minimizing \textbf{average travel time cost}. \changedF{Each cell shows the the average over 10 simulation
runs with different initial conditions.}}
\label{table:density_map}
\centering
\begin{tabular}{llllll}
\bottomrule
\multicolumn{2}{c|}{Using TDM (s)}          & \multicolumn{2}{c|}{Not Using TDM (s)}      & \multicolumn{2}{c}{Comparison (\%)} \\ \hline
Avg   & \multicolumn{1}{l|}{Max}    & Avg   & \multicolumn{1}{l|}{Max}    & Avg            & Max           \\ \hline
62.14 & \multicolumn{1}{l|}{124.83} & 67.14 & \multicolumn{1}{l|}{143.45} & \cellcolor[HTML]{D6F4F9}92.57\% & \cellcolor[HTML]{BAEFF6}87.02\% \\ \bottomrule
\end{tabular}
\end{table}

\begin{table}
\caption{Performance comparisons between simulations using path assignment probabilities computed using the \textbf{ensemble model} minimizing \textbf{max. travel time cost}, and using an uniform path assignment probability over 4 topological classes on map ``\emph{o2}'' (Figure \ref{pic_o2}). \changedF{Each cell shows the the average over 10 simulation
runs with different initial conditions.}}
\label{table:uniform}
\centering
\begin{tabular}{ll|ll|ll|ll}
\bottomrule
\multicolumn{1}{c}{\multirow{2}{*}{Robot \#}} & \multirow{2}{*}{Ped. \#} & \multicolumn{2}{c|}{Ensemble Model (s)} & \multicolumn{2}{c|}{Uniform Path Assignment (s)} & \multicolumn{2}{c}{Comparison (\%)} \\ \cline{3-8} 
\multicolumn{1}{c}{}                          &                          & Avg           & Max           & Avg           & Max          & Avg            & Max           \\ \hline
5                                             & 5                        & 43.01         & 61.81         & 49.61         & 72.94       & \cellcolor[HTML]{B8EFF5}86.70\% & \cellcolor[HTML]{AFEDF4}84.74\%   \\
10                                            & 10                       & 75.66         & 142.06        & 91.16         & 150.30      & \cellcolor[HTML]{A6EBF3}82.99\% & \cellcolor[HTML]{E0F6FB}94.51\%   \\
15                                            & 15                       & 106.53        & 196.84        & 120.04        & 237.92      & \cellcolor[HTML]{C3F1F7}88.74\% & \cellcolor[HTML]{A4EBF3}82.73\%   \\ \bottomrule
\end{tabular}
\end{table}

\begin{table}
\caption{Performance comparisons for multi-group scenarios (3 groups, 5 robots in each group) between the topological method and the shortest-path method in map ``\emph{group}'' (Figure \ref{pic_group}). The topological method uses \textbf{ensemble model} minimizing \textbf{max. time time cost} with 10 classes. \changedF{Each cell shows the the average over 10 simulation
runs with different initial conditions.}}
\label{table:group}
\centering
\begin{tabular}{ll|ll|ll}
\bottomrule
\multicolumn{2}{c|}{Topological (s)} & \multicolumn{2}{c|}{Shortest (s)} & \multicolumn{2}{c}{Comparison (\%)} \\ \hline
Avg               & Max              & Avg             & Max             & Avg              & Max              \\ \hline
59.7901           & 107.439          & 65.0912         & 115.988         & \cellcolor[HTML]{D2F4F9}91.86\% & \cellcolor[HTML]{D6F4F9}92.63\%   \\ \bottomrule
\end{tabular}
\end{table}

\vspace{0.5em}\noindent
\textbf{iii.}
Although most of our simulations focus on single-group scenarios, we have applied our algorithm to a multi-group case as well. In map ``\emph{group}'' (Figure~\ref{pic_group}), three groups of robots starting of from different locations try to reach their respective goals cross the map. \highlight{As suggested by Table \ref{table:group}, our method, with 10 classes for robots to choose from performs better than the shortest-path algorithm in a statistically significant manner.}
}

}

\vspace{1em}\noindent
\textbf{Real Robot Experiments:}
Real-robot experiments were run only on the \emph{cage\_2} map with 9 robots and 9 pedestrians (Fig \ref{fig:experiment_setup}).
% Each of the map is divided into topological classes by several obstacles. 
% The number of pedestrians and their trajectories are randomly generated (see accompanying video attachment). 
The results from each of the 10 runs are summarized in Table \ref{table:experiment_data}
and demonstrate similar advantages as seen in simulations.
\changedF{Complete video of the simulation can be found in the multimedia attachment.}
% prove this advantage, which is very close to the simulation counterpart.
\revisedXW{It is to be noted that the robots used in the experiments are omni-directional and are is allowed to stop and/or move back in order to avoid immediate collision with other robots or pedestrians that it can sense in its immediate neighborhood. Since the robots follow paths computed by A* planer on a discrete grid representation of the environment, it needs to follow a piece-wise path that is restricted to the graph.}
}}

\begin{SCfigure}[2][h] \vspace{1em}  %{figure}[t] \vspace{-1em}
\centering
\includegraphics[width=0.5\columnwidth]{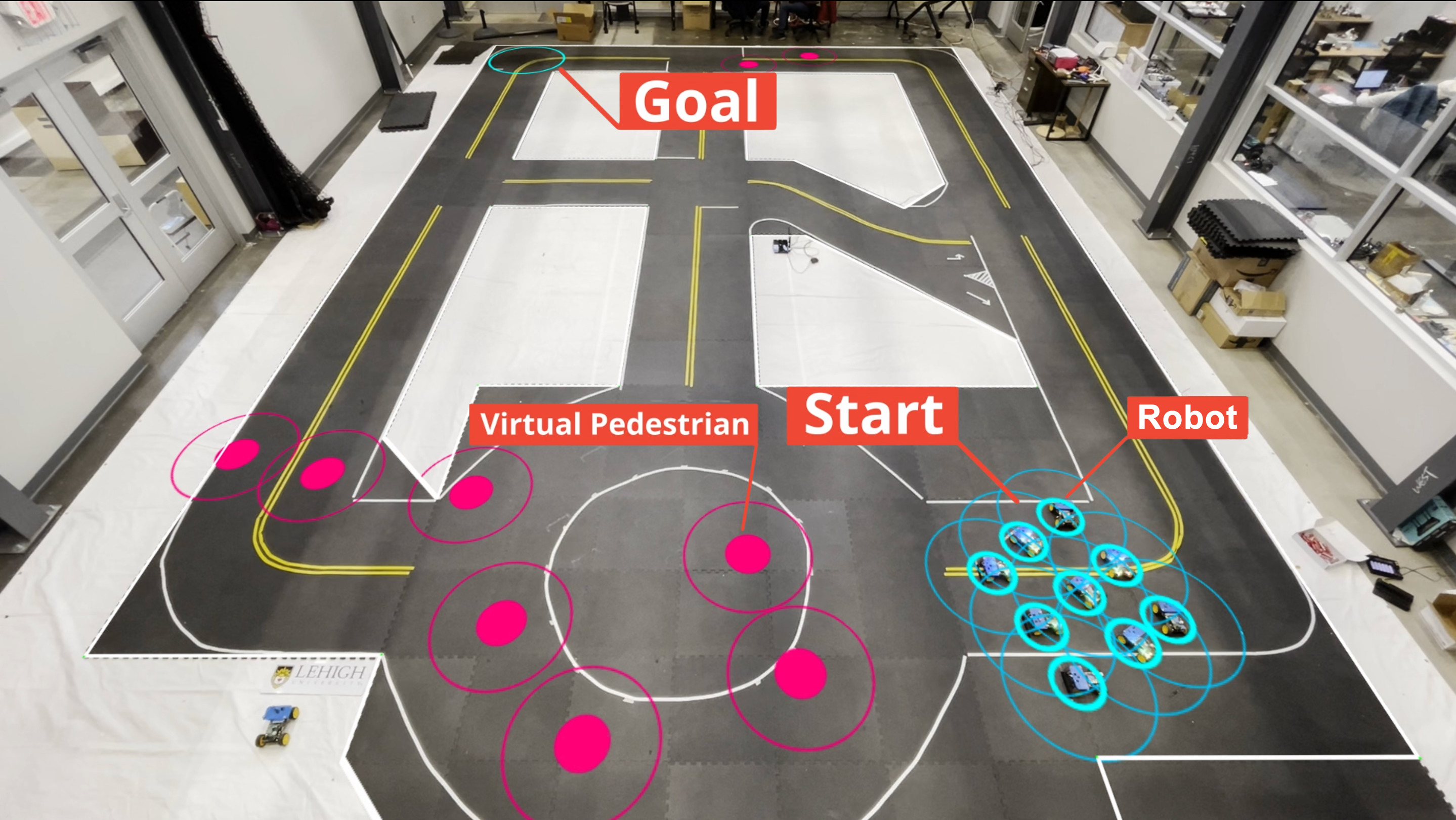}
%\vspace{-0.5em}
\caption{Real robot experiment in map ``\emph{cage\_2}''. Virtual (augmented reality) pedestrians are used in the experiment (see multimedia attachment).}
%\vspace{-2em}
\label{fig:experiment_setup}
\end{SCfigure} %{figure}

\begin{table}[ht]
\caption{Statistic from individual real-robot experiments on map ``\emph{cage\_2}''. All runs are with 9 robots and 9 virtual pedestrians, using the 2-robot model minimizing average travel time cost, $C_{\text{avg}}$.
% Comparing with the simulation results, the percentage of average travel time and that of max travel time with this setup are close to the counterparts in Table \ref{tab_performance_comparison}.
}
\label{table:experiment_data}
\centering
%\resizebox{\columnwidth}{!}{
\begin{threeparttable}

\begin{tabular}{rrrrrrr|rrr}
\bottomrule
\multicolumn{1}{r|}{}  & \multicolumn{3}{c|}{Topological (s)} & \multicolumn{3}{c|}{Shortest (s)} & Avg & Max & Coll\\
\cline{2-7}
\multicolumn{1}{r|}{\multirow{-2}{*}{Expt.\#}} & Avg & Max & \multicolumn{1}{r|}{Coll} & Avg & Max & Coll& (ratio) & (ratio) & (diff.) \\ \hline
\multicolumn{1}{l|}{1}  & 56.82 & 86.20  & \multicolumn{1}{l|}{0.00} & 86.93  & 143.14 & 0.02 & \cellcolor[HTML]{4DDAE7}65.37\% & \cellcolor[HTML]{33D5E3}60.22\%  & \cellcolor[HTML]{00CCDD}-0.02 \\
\multicolumn{1}{l|}{2}  & 70.04 & 120.27 & \multicolumn{1}{l|}{0.20} & 72.76  & 116.96 & 0.19 & \cellcolor[HTML]{E9F8FC}96.26\% & \cellcolor[HTML]{FCEDF3}102.83\% & \cellcolor[HTML]{FF8DAA}0.01  \\
\multicolumn{1}{l|}{3\tnote{*}}  & 76.32 & 94.45  & \multicolumn{1}{l|}{0.17} & 111.68 & 162.72 & 0.70 & \cellcolor[HTML]{5CDDE9}68.34\% & \cellcolor[HTML]{28D3E2}58.04\%  & \cellcolor[HTML]{35D6E4}-0.53 \\
\multicolumn{1}{l|}{4}  & 60.69 & 116.82 & \multicolumn{1}{l|}{0.00} & 106.08 & 138.78 & 0.68 & \cellcolor[HTML]{24D2E1}57.21\% & \cellcolor[HTML]{ACECF4}84.18\%  & \cellcolor[HTML]{00CCDD}-0.68 \\
\multicolumn{1}{l|}{5}  & 57.32 & 85.92  & \multicolumn{1}{l|}{0.04} & 79.75  & 96.77  & 0.37 & \cellcolor[HTML]{6EE0EB}71.87\% & \cellcolor[HTML]{C3F1F7}88.79\%  & \cellcolor[HTML]{84E5EE}-0.32 \\
\multicolumn{1}{l|}{6}  & 57.69 & 99.05  & \multicolumn{1}{l|}{0.43} & 98.10  & 201.87 & 0.44 & \cellcolor[HTML]{2CD4E2}58.81\% & \cellcolor[HTML]{00CCDD}49.07\%  & \cellcolor[HTML]{F7FBFE}-0.01 \\
\multicolumn{1}{l|}{7}  & 46.83 & 55.03  & \multicolumn{1}{l|}{0.34} & 90.86  & 132.60 & 0.61 & \cellcolor[HTML]{07CDDE}51.55\% & \cellcolor[HTML]{00CCDD}41.50\%  & \cellcolor[HTML]{98E9F1}-0.27 \\
\multicolumn{1}{l|}{8}  & 51.59 & 82.63  & \multicolumn{1}{l|}{0.07} & 83.15  & 105.48 & 0.43 & \cellcolor[HTML]{3CD7E5}62.04\% & \cellcolor[HTML]{8EE7F0}78.34\%  & \cellcolor[HTML]{73E2EC}-0.37 \\
\multicolumn{1}{l|}{9}  & 63.45 & 81.31  & \multicolumn{1}{l|}{0.20} & 86.74  & 129.86 & 0.54 & \cellcolor[HTML]{74E2EC}73.15\% & \cellcolor[HTML]{3FD8E5}62.61\%  & \cellcolor[HTML]{7BE3ED}-0.34 \\
\multicolumn{1}{l|}{10} & 65.42 & 140.28 & \multicolumn{1}{l|}{0.29} & 130.53 & 169.16 & 0.67 & \cellcolor[HTML]{00CCDD}50.12\% & \cellcolor[HTML]{A5EBF3}82.92\%  & \cellcolor[HTML]{6FE1EC}-0.38 \\ \hline
& & & & & & Avg & \cellcolor[HTML]{4DDAE7}65.47\% & \cellcolor[HTML]{69E0EB}70.85\%  & \cellcolor[HTML]{00CCDD}-0.29  \\
\toprule
\end{tabular}
\begin{tablenotes}
\item[*] See this run in the supplementary video. %There is a slight travel time difference due to the timer delay on camera. 
\item[**] All experiments use $a= 0.0001$, $b= 0.7222$.
\end{tablenotes} % \vspace{-3em}
\end{threeparttable}
%}
\end{table}

\clearpage

\revisedXW{
\vspace{0.5em}\noindent
\textbf{Discussions, Limitations and Future Directions:}
As demonstrated in the simulations and the experiments, our inter-agent coordination-free method distributing the agents across different routes outperforms other coordination-free methods with respect to the overall travel time. The advantage is particularly amplified when there are a large number of robots in the environment.
Compared to other MAPF methods, our method does not require real-time traffic information of both in- and out-of-system agents. The probability computation time using one of the simplified models for each robot is constant irrespective of the number of agents in the environment, and hence the computation complexity per robot does not increase with the number of agents, making our algorithm suitable for an environment with a large number of agents.
% In the simulations, the agents set off at the same time, while in the real world fewer agents do that. But, as long as there are enough agents running in the system using this method, the distribution of other agents would be even across different routes and stay constant, so it does not matter when a new agent sets  off and joins the system.
% Our method can be improved by integrating it with other methods, such as those for traffic density estimation, pedestrian position prediction, and local planning. 

Despite the demonstrated effectiveness of the proposed method, we recognize several limitations of the current method, which 
% \noindent
% The limitations our method 
warrant further future investigations:
\begin{enumerate}[i]
\item We use a priory traffic density estimation in the cost function for computing the reference paths for each robot. Currently this traffic density is computed synthetically purely based on the structure/map of the environment. In a real urban environment, historic traffic date can provide more accurate traffic density. In future we plan to test the proposed method with real traffic data collected from the department of transportation for constructing the traffic density map.
\item It is assumed that each robot has a priori knowledge of the environment (a map) and also knows its own location (using a global localization system such as GPS). Without one or both  of these information, each robot will also need to simultaneously create a map of the environment and/or localize itself in the environment. This will require each robot to use a SLAM (Simultaneous Localization and Mapping)~\cite{jia2019survey} module on top of our coordination-free planning algorithm, which we will do in the future.
\item We have used A* search algorithm in a discrete graph representation of the environment for computing the cost-minimizing paths restricted to the graph. This results in the individual robots following paths that are piecewise linear, but may have sudden turns because of the discrete nature of the graph. In future we will use any-angle planning algorithms such Theta*~\cite{daniel2010theta} or S*~\cite{bhattacharya2019towards} to generate smoother paths for the individual robots to follow.
% 
% \item
% Planning robots need to have a relatively close estimation of the total number of agents in the environment, while in reality, the estimation may not be accurate. It may be difficult to differentiate the non-rational agents from the rational agents in the real world, although we gave each agent exact numbers in the simulations and the experiment.
% \item
% It also requires the planning robots to have a one-time comprehensive study of the environment in advance to calculate the probability. It does not work for environment exploration situations.
% \item
% Since the proposed method is coordination-free, when competing methods with coordination, our method may still be disadvantaged.
% \item
% Our method is more suitable for road networks with potential traffic slowdowns at narrow passages. Open ground is not the situation it is designed for.
% \item
% This method does not guarantee a shorter travel time for every individual agent but leads to a more efficient overall performance, which could restrict its scope of use. It is more likely to be applied to autonomous agents than to manned agents. However, the situation that all robots cannot communicate but have this same stochastic strategy could be uncommon.
\end{enumerate}
}

 \clearpage
 
\bmhead{Acknowledgments}

%Acknowledgments are not compulsory. Where included they should be brief. Grant or contribution numbers may be acknowledged.
This material is based upon work supported by the National Science Foundation under Grant No. CCF-2144246.

\bibliography{references}

\end{document}